\documentclass[twoside,11pt]{article}

\usepackage{jmlr2e}

\usepackage{times}
\usepackage{graphicx}
\usepackage{subcaption}
\usepackage{amsfonts}
\usepackage{amsmath,amssymb}
\usepackage{bbm}
\usepackage{xcolor}
\usepackage[square,numbers]{natbib}

\renewcommand{\P}{\mathbf{P}}

\newcommand{\E}{\mathbf{E}}
\newtheorem{assumption}{Assumption}

\usepackage{placeins}

\DeclareMathOperator{\trace}{Tr}

\usepackage{enumerate}

\ShortHeadings{Optimal Statistical Rates for Decentralised Regression with Linear Speed-Up}{Richards and Rebeschini}
\firstpageno{1}

\begin{document}

\title{Optimal Statistical Rates for Decentralised Non-Parametric\\Regression with Linear Speed-Up}

\author{\name Dominic Richards \email dominic.richards@spc.ox.ac.uk \\
       \addr Department of Statistics\\
       University of Oxford\\
       24-29 St Giles', Oxford, OX1 3LB
       \AND
       \name Patrick Rebeschini \email patrick.rebeschini@stats.ox.ac.uk \\
       \addr Department of Statistics\\
       University of Oxford\\
       24-29 St Giles', Oxford, OX1 3LB
}

\editor{}

\maketitle

\thispagestyle{plain}

\begin{abstract}%
We analyse the learning performance of Distributed Gradient Descent in the context of multi-agent decentralised non-parametric regression with the square loss function when i.i.d.\ samples are assigned to agents. We show that if agents hold sufficiently many samples with respect to the network size, then Distributed Gradient Descent achieves optimal statistical rates with a number of iterations that scales, \emph{up to a threshold}, with the inverse of the spectral gap of the gossip matrix divided by the number of samples owned by each agent raised to a problem-dependent power. The presence of the threshold comes from statistics. It encodes the existence of a ``big data'' regime where the number of required iterations does \emph{not} depend on the network topology. In this regime, Distributed Gradient Descent achieves optimal statistical rates with the \emph{same} order of iterations as gradient descent run with all the samples in the network. Provided the communication delay is sufficiently small, the distributed protocol yields a \emph{linear} speed-up in runtime compared to the single-machine protocol. This is in contrast to decentralised optimisation algorithms that do not exploit statistics and only yield a linear speed-up in graphs where the spectral gap is bounded away from zero. Our results exploit the statistical concentration of quantities held by agents and shed new light on the interplay between statistics and communication in decentralised methods. Bounds are given in the standard non-parametric setting with source/capacity assumptions.
\end{abstract}

\section{Introduction}
In machine learning a canonical goal is to use training data sampled independently from an unknown distribution to fit a model that performs well on unseen data from the same distribution. With a loss function measuring the performance of a model on a data point, a common approach is to find a model that minimises the average loss on the training data with some form of \emph{explicit} regularisation to control model complexity and avoid overfitting. Due to the increasingly large size of datasets and high model complexity, direct minimisation of the regularised problem is posing more and more computational challenges. This has led to growing interest in approaches that improve models incrementally using gradient descent methods \citep{bousquet2008tradeoffs}, where model complexity is controlled through forms of \emph{implicit/algorithmic} regularisation such as early stopping and step-size tuning \citep{yao2007early,ying2008online,lin2017optimal}. 
 
The growth in the size of modern datasets has also meant that the coordination of multiple machines is often required to fit machine learning models.  
In the centralised server-clients setup, a single machine (server) is responsible to aggregate and disseminate information to other machines (clients) in what is an effective star topology. In some settings, such as ad-hoc wireless and peer-to-peer networks, network instability, bandwidth limitation and privacy concerns make centralised approaches less feasible. This has motivated research into scalable methods that can avoid the bottleneck and vulnerability introduced by the presence of a central authority. Such solutions are called ``decentralised'', as no single entity is responsible for the collection and dissemination of information: machines communicate with neighbours in a network structure that encodes communication channels. 

Since the early works \cite{tsitsiklis1986distributed,tsitsiklis1984problems} to the more recent work \cite{johansson2007simple,nedic2009distributed,Nedic2009,johansson2009randomized,lobel2011distributed,
matei2011performance,Boyd:2011:DOS:2185815.2185816,DAW12,shi2015extra,mokhtari2016dsa}, problems in decentralised multi-agent optimisation have often been treated as a particular instance of consensus optimisation. In this framework, a network of machines or agents collaborate to minimise the average of functions held by individual agents, hence ``reaching consensus'' on the solution of the global problem. In this setting the performance of the chosen protocol naturally depends on the network topology, since to solve the problem each agent \emph{has to} communicate and receive information from all other agents. In particular, the number of iterations required by decentralised iterative gradient methods typically scales with the inverse of the spectral gap of the communication matrix (a.k.a. gossip or consensus matrix) \citep{DAW12,scaman2017optimal,scaman2018optimal}, which reflects the performance of gossip protocols in the problem of distributed averaging \citep{boyd2006randomized,dimakis2008geographic,NET-014,benezit2010order}. 

Many distributed machine learning problems, in particular those involving empirical risk minimisation, have been framed in the context of consensus optimisation. However, as highlighted in \cite{shamir2014distributed} and more recently in \cite{2018arXiv180906958R}, often these problems have more structure than consensus optimisation due to the statistical regularity of the data.
When the agents' functions are the empirical risk of their local data, in the setting where the local data comes from the \emph{same} unknown distribution (homogeneous setting), the functions held by each agent are similar to one another by the phenomenon of statistical concentration. In particular, in the limit of an infinite amount of data per agent, the local functions are the same and agents do \emph{not} need to communicate to solve the problem. This phenomenon highlights the existence of a natural trade-off between statistics and communication. While statistical similarities of local objective functions and the statistics/communication trade-off have been investigated and exploited in centralised server-clients setup, typically in the analysis and design of divide-and-conquer schemes 
\citep{zhang2015divide,lin2017distributed,
guo2017learning,mucke2018parallelizing,lin2018optimal,agarwal2011distributed,zhang2012communication,
shamir2014distributed,shamir2014fundamental,zhang2015disco,arjevani2015communication}, only recently there has been some investigation into the interplay between statistics and communication/network-topology in the decentralised setting. 
The authors in \cite{bijral2017data} investigate the interplay between the spectral norm of the data-generating distribution and the inverse spectral gap of the communication matrix for Distributed Stochastic Gradient Descent in the case of strongly convex losses. As most of the literature on decentralised machine learning, this work also focuses on minimising the training error and not the test/prediction error (numerical experiments are given for the test error).
Some works have investigated the performance on the test loss in the single-pass/online stochastic setting where agents use each data point only once. The authors in \cite{Rabbat15,tsianos2016efficient} investigate a distributed regularised online learning setting \citep{xiao2010dual} and obtain guarantees for a ``multi-step'' Distributed Stochastic Mirror Descent algorithm where agents reach consensus on their stochastic gradients in-between computation steps. 
The works \cite{lian2017can} and \cite{assran2018stochastic} consider the performance of Distributed Stochastic Gradient Descent algorithms in the non-convex smooth case. They investigate the average performance of the agents over the network in terms of convergence to a stationary point of the test loss \citep{ghadimi2013stochastic} and show that a linear speed-up in computational time can be achieved provided the number of samples seen, equivalently the number of iterations performed, exceeds the network size times the inverse of the spectral gap, each raised to a certain power.
The work \cite{2018arXiv180906958R} seems to be the first to have considered minimisation of the test error in the multi-pass/offline stochastic setting that more naturally relates to the classical literature on consensus optimisation. The authors investigate stability of Distributed Stochastic Gradient Descent on the test error and show that for smooth and convex losses the number of iterations required to achieve optimal statistical rates scales with the inverse of the spectral gap of the gossip matrix, a term that captures the noise of the gradients' estimates, and a term that controls the statistical proximity of the local empirical losses. 

\subsection{Contributions}
In this work we investigate the implicit-regularisation learning performance of full-batch Distributed Gradient Descent \citep{Nedic2009} on the test error in the context of non-parametric regression with the square loss function. In the homogeneous setting where agents hold independent and identically distributed data points, we investigate the choice of step size and number of iterations that guarantee each agent to achieve optimal statistical rates with respect to all the samples in the network. We build a theoretical framework that allows to directly and explicitly exploit the statistical concentration of quantities (i.e.\ batched gradients) held by agents. On the one hand, exploiting concentration yields savings on computation, i.e.\ it allows to achieve faster convergence rates compared to methods that do not exploit concentration in their parameter tuning. On the other hand, it yields savings on communication, as it allows to take advantage of the trade-off between statistical power and communication costs. Firstly, we show that if agents hold sufficiently many samples with respect to the network size, then Distributed Gradient Descent achieves optimal statistical rates up to poly-logarithmic factors with a number of iterations that scales with the inverse of the spectral gap of the communication matrix divided by the number of samples owned by each agent raised to a problem-dependent power, up to a statistics-induced threshold. Previous results for decentralised iterative gradient schemes in the context of consensus optimisation do not take advantage of the statistical nature of decentralised empirical risk minimisation problems. In the statistical setting that we consider, these methods would require a larger number of iterations that scales only with respect to the inverse of the spectral gap. Secondly, we show that if agents additionally hold sufficiently many samples with respect to the inverse of the spectral gap, then the \emph{same} order of iterations allows Distributed Gradient Descent and Single-Machine Gradient Descent (i.e.\ gradient descent run on a single machine that holds all the samples in the network) to achieve optimal statistical rates up to poly-logarithmic factors. Provided the communication delay is sufficiently small, this yields a \emph{linear} speed-up in runtime over Single-Machine Gradient Descent, with a ``single-step'' method that performs a single communication round per local gradient descent step. Single-step methods that do not exploit concentration can only achieve a linear speed-up in runtime in graphs with spectral gap bounded away from zero, i.e.\ expanders or the complete graph. Our results demonstrate how the increased statistical similarity between the local empirical risk functions can make up for a decreased connectivity in the graph topology, showing that a linear speed-up in runtime can be achieved in \emph{any} graph topology by exploiting concentration. 
To the best of our knowledge, we seem to be the first to isolate this type of phenomena.

We prove our results under the standard ``source'' and ``capacity'' assumptions in non-parametric regression. These assumptions relate, respectively, to the projection of the optimal predictor on the hypothesis space and to the effective dimension of this space \citep{zhang2005learning,caponnetto2007optimal}. A contribution of this work is to show that proper tuning yields, up to poly-logarithmic terms, optimal non-parametric rates in decentralised learning. As far as we aware, in the distributed setting such guarantees have been established only for centralised divide-and-conquer methods \citep{zhang2015divide,lin2017distributed,
guo2017learning,mucke2018parallelizing,lin2018optimal}.

To prove our results we build upon previous work for Single-Machine Gradient Descent applied to non-parametric regression, in particular the line of works \cite{yao2007early,rosasco2015learning,lin2017optimal}.
Exploiting that in our setting the iterates of Distributed Gradient Descent can be written in terms of products of linear operators depending on the data held by agents, we decompose the excess risk into bias and sample variance terms for Single-Machine Gradient Descent plus an additional quantity that captures the error incurred by using a decentralised protocol over the communication network. We analyse this network error term by further decomposing it into a term that behaves similarly to the consensus error previously considered in \cite{DAW12,Nedic2009}, and a new higher-order term. We control both terms by using the structure of the gradient updates, which allows us to analyse the interplay between statistics, via concentration, and network topology, via mixing of random walks related to the gossip matrix.

The work is structured as follows. Section \ref{sec:DecentralisedLearning} presents the setting, assumptions, and algorithm that we consider. Section \ref{sec:Results:Main} states the main convergence result and discusses implications from the point of view of statistics, computation and communication. Section \ref{sec:Results:Detailed} presents the error decomposition into bias, variance, and network error, and it illustrates the implicit regularisation strategy that we adopt. Section \ref{sec:Conclusion} highlights some of the features of our contribution in the light of future research directions. The appendix in the supplementary material is structured as follows. Section \ref{sec:remarks} includes some remarks about our results. Section \ref{sec:ProofSketch} illustrates the main scheme of the proofs, highlighting the interplay between statistics and network topology. Section \ref{Appendix:Proofs} contains the full details of the proofs.

\section{Setup}
\label{sec:DecentralisedLearning}
In this section we describe the learning problem, assumptions and algorithm that we consider.

\subsection{Learning problem: decentralised non-parametric least-squares regression}
\label{sec:DecentralisedLearning:Setup}
We adopt the setting used in \cite{rosasco2015learning,lin2017optimal}, which involves regression in abstract Hilbert spaces. This setting is of relevance for applications related to the Reproducing Kernel Hilbert Space (RKHS). See the work in \cite{yao2007early} and references therein.

Let $H$ be a separable Hilbert Space with inner product and induced norm denoted by $\langle \,\cdot\,,\,\cdot\, \rangle_{H}$ and $\|\cdot\|_{H}$, respectively. Let $X \subseteq H$ be the input space and $Y \subset \mathbb{R}$ be the output space. Let $\rho$ be an unknown probability measure on $Z = X \times Y$, $\rho_{X}(\,\cdot\,)$ be the marginal on $X$, and $\rho(\,\cdot\,|x)$ be the conditional distribution on $Y$ given $x \in X$. Assume that there exists a constant $\kappa \in [1,\infty)$ so that 
\begin{align}
\label{Assumption:BoundedProduct}
	\langle x,x^{\prime}\rangle_{H} \leq \kappa^2, \quad 
	\forall x, x^{\prime} \in X. 
\end{align}
Let the network of agents be modelled by a simple, connected, undirected, finite graph $G=(V,E)$, with $|V|=n$ nodes joined by edges $E \subseteq V \times V$. Edges represent communication constraints: agents $v,w \in V$ can only communicate if they share an edge $(v,w) \in E$. We consider the homogeneous setting where each agent $v \in V$ is given $m$ data points $\mathbf{z}_{v} := \{\mathbf{x}_{v},\mathbf{y}_{v} \}$  sampled independently from $\rho$, where  
$\mathbf{x}_v = \{x_{i,v}\}_{i=1,\dots,m}$ and $\mathbf{y}_v = \{y_{i,v}\}_{i=1,\dots,m}$, and each pair $(x_{i,v},y_{i,v})$ is sampled from $\rho$.
The problem under study is the minimisation of the test/prediction risk with the square loss:
\begin{align}
\label{Equ:Objective}
\inf_{\omega \in H} \mathcal{E}(\omega), \qquad 
\mathcal{E}(\omega) 
= \int_{X \times Y} (\langle \omega,x\rangle_{H} - y)^2 d\rho(x,y),
\end{align}
The quality of an approximate solution $\widehat{\omega} \in H$ is measured by the excess risk
$
\mathcal{E}(\widehat{\omega}) - \inf_{\omega \in H} \mathcal{E}(\omega).
$
\paragraph{Notation}
Given a matrix $A \in \mathbb{R}^{n \times n}$, let $A_{vw}$ denote the 
$(v,w)\text{-th}$ element and $A_{v} = (A_{vw})_{w=1,\dots,n}$ denote the $v \text{-th}$ row. Let $O(\,\cdot\,)$ denote orders of magnitudes up to constants in $n$ and $m$, and  $\widetilde{O}(\,\cdot\,)$ denote orders of magnitudes up to both constants and poly-logarithmic terms in $n$ and $m$. Let $\lesssim,\gtrsim,\simeq$ denote inequalities and equalities modulo constants and poly-logarithmic terms in $n, m$. We use the notation $a \vee b = \max\{a,b\}$ and $a \wedge b = \min\{a,b\}$.

\subsection{Assumptions}
\label{sec:Results:Assumptions}

The assumptions that we consider are standard in non-parametric regression \citep{lin2017optimal,pillaud2018statistical}. The first assumption is a control on the even moments of the response.
\begin{assumption}
\label{Assumption:Moments}
There exist $M \in (0,\infty)$ and $\nu \in (1,\infty)$ such that for any $\ell \in \mathbb{N}$ we have \\ $
\int_{Y} y^{2\ell} d\rho(y|x) \leq \nu \ell! M^{\ell}
\text{ $\rho_{X}$-almost surely}.
$
\end{assumption}

Let $L^2(H,\rho_{X})$ be the Hilbert space of square-integrable functions from $H$ to $\mathbb{R}$ with respect to $\rho_{X}$, with norm $
\|f\|_{\rho} := ( 
\int_{X} |f(x)|^2 d\rho_{X}(x)
)^{1/2}
$.
Let $\mathcal{L}_{\rho} : L^{2}(H,\rho_{X}) \rightarrow L^{2}(H,\rho_{X})$ be the operator defined as $\mathcal{L}_{\rho}(f) := \int_{X} \langle x,\,\cdot\, \rangle_{H} f(x) d\rho_{X}(x)$. 
Under Assumption \ref{Assumption:BoundedProduct} the operator $\mathcal{L}_{\rho}$ can be proved to be in the class of positive trace operators \citep{cucker2007learning}, and therefore the $r$-th power $\mathcal{L}^{r}_{\rho}$, with $r \in \mathbb{R}$, can be defined by using spectral theory. Let us also define the operator $\mathcal{T}_{\rho}:H \rightarrow H$ as $\mathcal{T}_{\rho} := \int_{X} \langle x,\,\cdot\,\rangle_{H} x d\rho_{X}(x)$ and its operator norm $\|\mathcal{T}_\rho\| := \sup_{\omega \in H, \|\omega\|_{H} = 1} \|\mathcal{T}_\rho \omega \|_{H}$. 
The function minimising the expected squared loss \eqref{Equ:Objective} over all measurable functions  $f:H \rightarrow \mathbb{R}$ is known to be the conditional expectation 
$f_{\rho}(x) := \int_{Y} y d\rho(y|x)$ for $x \in \mathcal{X}$.
Let $H_{\rho} := \{f : X \rightarrow \mathbb{R}\,|\,\exists \omega \in H \text{ with } f(x) = \langle w,x\rangle_{H}, \rho_{X}\text{-almost surely}\}$ be the hypothesis space that we consider. The optimal $f_{\rho}$ may not be in $H_{\rho}$ as under Assumption \ref{Assumption:BoundedProduct} the space of functions searched $H_{\rho}$ is a subspace of $L^{2}(H,\rho_{X})$. Let $f_{H}$ denote the projection of $f_{\rho}$ onto the closure of $H_{\rho}$ in $L^{2}(H,\rho_{X})$. Searching for a solution to \eqref{Equ:Objective} is equivalent to searching for a linear function in $H_{\rho}$ that approximates $f_{H}$.\\
The following assumption quantifies how well the target function $f_{H}$ can be approximated in $H_{\rho}$. 
\begin{assumption}
\label{Assumption:Source}
There exist $r > 0$ and $R > 0$ such that $\|\mathcal{L}_{\rho}^{-r} f_{H}\|_{\rho} \leq R$.
\end{assumption}
This assumption is often called the ``source'' condition \citep{caponnetto2007optimal}. Representing $f_{H}$ in the eigenspace of $\mathcal{L}_{\rho}$, this condition can be related to the rate at which the coefficients of this representation decay. The bigger $r$ is, the faster the decay, and more stringent the assumption is. In particular, if $r \geq 1/2$ then the target function is in the hypothesis space $f_{H} \in H_{\rho}$. 
The last assumption is on the capacity of the hypothesis space.
\begin{assumption}
\label{Assumption:Capacity}
There exist $\gamma \in (0,1],c_{\gamma} > 0$ such that 
$
\trace ( \mathcal{L}_{\rho}\big( \mathcal{L}_{\rho} \!+\! \lambda I\big)^{-1}) \leq
c_{\gamma} \lambda^{-\gamma}$ for all $\lambda > 0$.
\end{assumption}
Assumption \ref{Assumption:Capacity} relates to the effective dimension of the underlying regression problem  \citep{zhang2005learning,caponnetto2007optimal} and is often called the ``capacity'' assumption. This assumption is always satisfied for $\gamma =1$ and $c_{\gamma} = \kappa^2$ since $\mathcal{L}_{\rho}$ is a trace class operator. This case is called the capacity-independent setting. Meanwhile, this assumption is satisfied for $\gamma \in (0,1]$ if, for instance, the eigenvalues of $\mathcal{L}_\rho$, denoted by $\{\tau_i\}_{i \geq 1}$, decay sufficiently quickly, i.e.\  $\tau_i = O(i^{-1/\gamma})$.  This case allows improved rates to be obtained.
For more details on the interpretation of these assumptions we refer to the work in
\cite{rosasco2015learning,lin2017optimal,pillaud2018statistical}.

\subsection{Algorithm: distributed gradient descent}
\label{sec:DecentralisedLearning:DistributedGD}
We now describe the Distributed Gradient Descent 
 algorithm \citep{Nedic2009} and its application to the problem of non-parametric regression. Let $P \in \mathbb{R}^{n \times n}_{\geq 0}$ be a symmetric doubly-stochastic matrix, i.e.\ $P = P^{\top}$ and $P \mathbf{1} = \mathbf{1}$ where $\mathbf{1} = (1,\dots,1) \in \mathbb{R}^{n}$ is the vector of all ones. Let $P$ be supported on the graph, i.e.\ for any $v\neq w$, $P_{vw} \not= 0 $ only if $(v,w) \in E$. 
 The matrix $P$ encodes local averaging on the network: when each agent has a real number represented by the vector $a = (a_{v})_{v \in V} \in \mathbb{R}^{n}$, the vector $(P a)_{v} = \sum_{w \in V} P_{vw} a_{w}$ for $v \in V$ encodes what each agent computes after taking a weighted average of its own and neighbours' numbers. 
Distributed Gradient Descent is implemented by communication on the network through the gossip matrix $P$. Initialised at $w_{1,v} = 0$ for $v \in V$, the iterates of the Distributed Gradient Descent  are defined as follows, for $v \in V$ and $t \geq 1$:
\begin{align}
\label{alg:DistributedGD}
\omega_{t+1,v} = 
\sum_{w \in V} 
P_{vw} 
\Big( 
\omega_{t,w} 
- 
\eta_{t} 
\frac{1}{m} 
\sum_{i=1}^{m}
\big(
\langle \omega_{t,w},x_{i,w} \rangle_H - y_{i,w}
\big)x_{i,w}
\Big),
\end{align}
where $\{\eta_{t}\}_{t \ge 1}$ is the sequence of positive step sizes. The iterates \eqref{alg:DistributedGD} can be seen as a combination of two steps: first, each agent $w \in V$ performs a local gradient descent step 
$\omega_{t+1/2,w} = \omega_{t,w} 
- 
\eta_{t} 
\frac{1}{m} 
\sum_{i=1}^{m}
\big(
\langle \omega_{t,w},x_{i,w}\rangle_H - y_{i,w} 
\big)x_{i,w}$; second, each agent performs local averaging 
through the consensus step\footnote{
We note, while this assumes agents communicate infinite dimensional quantities in the general non-parametric setting, the framework we consider accommodates finite approximations of infinite dimensional quantities whilst accounting for the statistical precision \cite{carratino2018learning}.
}
$\omega_{t+1,v} = \sum_{w \in V} P_{vw}\omega_{t+1/2,w}$. We treat gradient descent as a statistical device. We are interested in tuning the parameters of the algorithm to bound the expected value of the excess risk $\E[\mathcal{E}(\omega_{t+1,v})] - \inf_{\omega \in H} \mathcal{E}(\omega),$ where $\E[\,\cdot\,]$ denotes expectation with respect to the data $\{\mathbf{z}_{v}\}_{v\in V}$.

\paragraph{Network dependence}
Let $\sigma_2$ be the second largest eigenvalue in magnitude of the communication matrix $P$. Specifically, given the spectral decomposition of the gossip matrix $P = \sum_{l=1}^{n} \lambda_{l} u_{l} u_{l}^{\top}$ where $1= \lambda_1 \geq \lambda_2 \geq,\dots,\geq \lambda_n >-1$ are the ordered real eigenvalues of $P$ and $\{u_{l}\}_{l=1,\dots,n}$ the associated eigenvectors, we have $\sigma_2 := \max\{|\lambda_2|,|\lambda_{n}|\}$. 
In many settings, the spectral gap scales with the size of the network raised to a certain power depending on the topology. For instance, supposing $G$ is a finite regular graph and the communication matrix is the random walk matrix, then the inverse of the spectral gap $(1-\sigma_2)^{-1}$ scales as $\Theta(1)$ for a complete graph, $\Theta(n)$ for a grid, and $\Theta(n^2)$ for a cycle \citep{chung1997spectral,levin2017markov,DAW12}. The question of designing gossip matrices $P$ that yield better (smaller) scaling for the quantity $(1-\sigma_2)^{-1}$ has been investigated \citep{Xiao2004}, and it has been found numerically that the rates mentioned above can not be improved unless lifted graphs are considered \citep{NET-014}.

\section{Main result: optimal statistical rates with linear speed-up in runtime}
\label{sec:Results:Main}
We now state and highlight the main contribution of this work in the context of decentralised statistical optimisation. The result that we are about to state in Theorem \ref{Cor:Main} showcases the interplay between statistics and communication that arise from the statistical regularities of the problem. This result shows the existence of a ``big data'' regime where Distributed Gradient Descent can achieve a linear (in the number of agents $n$) speed-up in runtime compared to Single-Machine Gradient Descent.

\begin{theorem}
\label{Cor:Main}
Let Assumptions \ref{Assumption:Moments}, \ref{Assumption:Source}, \ref{Assumption:Capacity} hold with $r \geq 1/2$ and  $2r+\gamma > 2$. Let $t$ be the smallest integer greater than the quantity
\begin{align*}
    \underbrace{ (nm)^{1/(2r+\gamma)} }_{\text{Single-Machine Iterations}} 
    \times \begin{cases}
     \Big( \frac{(nm)^{2r/(2r+\gamma)}}{ m (1-\sigma_2)^{\gamma} } \Big)^{1/\gamma}  
      \vee 1  & \text{ if } m \geq n^{2r/\gamma} \\
     \frac{(nm)^{r/(2r+\gamma)}}{\sqrt{m}(1-\sigma_2)}
     & \text{otherwise}
    \end{cases}
\end{align*}
Let $\eta_s \equiv \eta = \frac{\kappa^{-2} (nm)^{1/(2r+\gamma)}}{t}$ $\forall s\ge 1$.
If $m \geq n^{\frac{2r + 2 +\gamma}{2r +\gamma -2}}$ and $n \geq 2 (1+r) 
\log ( \frac{n}{1-\sigma_2} )$, then $\forall v \in V$:
\begin{align*}
& \E[\mathcal{E}(\omega_{t+1,v})] 
- \inf_{\omega \in H} \mathcal{E}(\omega)
\leq 
C (nm)^{-2r/(2r+\gamma)},
\end{align*}
where $C$ depends on  $\kappa^2,\|\mathcal{T}_{\rho}\|,M,\nu,r,R,\gamma,c_{\gamma}$, and polynomials of $\log(nm)$ and $\log(\frac{1}{1-\sigma_2})$. 
\end{theorem}

Theorem \ref{Cor:Main} shows that when agents are given sufficiently many samples ($m$) with respect to the number of agents ($n$), $m \geq n^{\frac{2r +2 +\gamma}{2r + \gamma -2}} $, proper tuning of the step size and number of iterations (a form of implicit regularisation) allows Distributed Gradient Descent to recover the optimal statistical rate $O((nm)^{-2r/(2r+\gamma)})$ for $r \in (1/2,1)$ \citep{caponnetto2007optimal} up to poly-logarithmic terms. 

Single-Machine Gradient Descent run on all of the observations has been previously shown to reach optimal statistical accuracy with a number of iterations of the order \\ 
$t_{\text{Single-Machine}} \sim O((nm)^{1/(2r+\gamma)})$ \citep{lin2017optimal}. The number of iterations $t\equiv t_{\text{Distributed}}$ prescribed by Theorem \ref{Cor:Main} scales like  $t_{\text{Single-Machine}}$ times a network-dependent factor that is a function of the inverse of the spectral gap $(1-\sigma_2)^{-1}$.
The fact that the number of iterations required to reach a prescribed level of error accuracy is inversely proportional to the spectral gap is a standard feature of iterative gradient methods applied to generic decentralised consensus optimisation problems \citep{DAW12,scaman2017optimal,scaman2018optimal}. This dependence encodes the fact that in the case of \emph{generic} objective functions assigned to agents, agents \emph{have to} share information with everyone to solve the global problem and minimise the sum of the local functions; hence, more iterations are required in graph topologies that are less well-connected. In the present homogeneous setting, however, the statistical nature of the problem allows to exploit concentration of random variables to characterise the existence of a (network-dependent) ``big data'' regime where the number of iterations does \emph{not} depend on the network topology. The trade-off between statistics and communication is encoded by the dependence of the tuning parameters (stopping time and step size) on the number of samples $m$ assigned to each agent. 
Observe that the factor 
$( \frac{(nm)^{2r/(2r+\gamma)}}{ m (1-\sigma_2)^{\gamma} })^{1/\gamma} \vee 1$  is a decreasing function of $m$, up to the threshold $1$. 
When $m \geq \frac{ n^{2r/\gamma}}{(1-\sigma_2)^{2r+\gamma}}  \vee n^{\frac{2r +2 +\gamma}{2r + \gamma -2}}$ this factor becomes $1$ and Theorem \ref{Cor:Main} guarantees that the \emph{same} order of iterations allows both Distributed and Single-Machine Gradient Descent to achieve the optimal statistical rates up to poly-logarithmic factors.  This regime represents the case when the increased statistical similarity between the local empirical risk functions assigned to each agent (increasing as a function of $m$, as described by the non-asymptotic Law of Large Numbers) makes up for the decreased connectivity in the graph topology (typically decreasing with the spectral gap $1-\sigma_2$) to yield a linear speed-up in runtime over Single-Machine Gradient Descent when the communication delay between agents is sufficiently small. See Section \ref{sec:LinearSpeedUp} below.

The result of Theorem \ref{Cor:Main} depends on some other requirements which we now briefly discuss. The requirement $n \geq 2 (1+r) \log ( \frac{n}{1-\sigma_2} )$ is technical and arises from the need to perform sufficiently many iterations to reach the mixing time of the gossip matrix $P$, i.e.\ $t \gtrsim (1-\sigma_2)^{-1}$.
Noting that the number of iterations $t$ depends on the number of agents, samples and spectral gap.
The requirement $2r+\gamma > 2$ relates to the difficulty of the estimation problem and is stronger than a similar condition seen for single-machine gradient methods where $2r+\gamma > 1$, see for instance the works \cite{lin2017optimal,pillaud2018statistical}. This requirement, alongside $m \geq n^{ \frac{ 2r+2+\gamma}{2r+\gamma -2}}$, ensures that the higher-order error terms arising from considering a decentralised protocol decay sufficiently quickly with respect to the number of samples owned by agents $m$. The condition $m \geq n^{ \frac{ 2r+2+\gamma}{2r+\gamma -2}}$ can be removed if the covariance operator $\mathcal{T}_{\rho}$ is assumed to be known to agents, which aligns with the additive noise oracle in single-pass Stochastic Gradient Descent \citep{dieuleveut2017harder} or fixed-design regression in finite-dimensional settings \cite{gyorfi2006distribution}.
The condition $m \ge n^{2r/\gamma}$ corresponds to the case when the rate of concentration of the batched gradients held by agents (i.e. $1/m$) is faster than the optimal statistical rate, i.e.\ $\frac{1}{m} \leq (nm)^{-2r/(2r+\gamma)}$.
This condition becomes more stringent  (i.e.\ more data per agent is needed) as the problem becomes easier from a statistical point of view and $r$ and $1/\gamma$ increase (see discussion in Section \ref{sec:Results:Assumptions}).
This is due to the fact that as $r$ and $1/\gamma$ increase, only the statistical rate improves while the rate of concentration in the network error stays the same, implying that more data is needed to balance the two terms.

\subsection{Linear speed-up in runtime}
\label{sec:LinearSpeedUp}
Let gradient computations cost $1$ unit of time and communication delay between agents be $\tau$ units of time.\footnote{
For details on this communication model as well as comparison to \cite{tsianos2012communication} see remarks within Appendix  \ref{sec:remarks}.
}
Denote the number of iterations required by Single-Machine Gradient Descent and Distributed Gradient Descent to achieve the optimal statistical rate by $t_{\text{Single-Machine}}$ and $t_{\text{Distributed}}$, respectively. The speed-up in computational time obtained by running the distributed protocol over the single-machine protocol is of the order $\frac{ t_{\text{Single-Machine}}}{t_{\text{Distributed}}} \frac{ nm}{m + \tau + \text{Deg}(P)}$, where $\text{Deg}(P)=\max_{v\in V} |\{P_{vw}\neq 0, w\in V\}|$ is the maximum degree of the communication matrix $P$.
Theorem \ref{Cor:Main} implies that when  $m \geq \frac{ n^{2r/\gamma}}{(1-\sigma_2)^{2r+\gamma}}  \vee n^{\frac{2r +2 +\gamma}{2r + \gamma -2}}$  then $t_{\text{Distributed}} \sim t_{\text{Single-Machine}}$, and if $\tau  + \text{Deg}(P) $ grows as $O(m)$ then the speed-up in computational time is of order $n$, linear in the number of agents. Classical ``single-step'' decentralised methods that alternate single communication rounds per local gradient computation, such as the methods inspired by \cite{Nedic2009}, do not exploit concentration and have a runtime that scales with the inverse of the spectral gap, without any threshold. As a result, these methods only yield a linear speed-up in graphs with spectral gap bounded away from zero, i.e.\ expanders or the complete graph. See below for more details. On the other hand, ``multi-step'' methods that alternate multiple communication rounds per local gradient computation, such as the ones considered in \cite{Rabbat15,tsianos2016efficient,scaman2017optimal,scaman2018optimal}, display a runtime that scales with a factor of the form $m+\frac{\tau+\text{Deg}(P)}{1-\sigma_2}$ in our setting. Thus, while these methods can achieve a linear speed-up in any graph topology in the ``big data'' regime $m\gtrsim\frac{\tau+\text{Deg}(P)}{1-\sigma_2}$ without exploiting concentration, they require an additional amount of communication rounds that is network-dependent and scales with the inverse of the spectral gap. For a cycle graph, for instance, this means an extra $O(n^2)$ communication steps per iteration (or $O(n)$ for gossip-accelerated methods). Hence, classical decentralised optimisation methods that do not exploit concentration suffer from a trade-off between runtime and communication cost: if you reduce the first you increase the second, and viceversa. 
Our results show that single-step methods can achieve a linear speed-up in runtime in \emph{any} graph topology by exploiting concentration: statistics allows to find a regime where it is possible to simultaneously have a linear speed-up in runtime without increasing communication.

\paragraph{Comparison to single-step decentralised methods that do not exploit concentration}
Decentralised optimisation methods that do not consider statistical concentration rates in their parameter tuning can not exploit the statistics/communication trade-off encoded by the presence of the factor $( \frac{(nm)^{2r/(2r+\gamma)}}{ m (1-\sigma_2)^{\gamma} })^{1/\gamma} \vee 1$ in Theorem \ref{Cor:Main}, and they typically require a smaller step size and more iterations to achieve optimal statistical rates. The convergence rate typically achieved by classical consensus optimisation methods, e.g. \cite{DAW12}, is recovered in Theorem \ref{Cor:Main} when  $m = n^{2r/\gamma}$ as in this case the number of iterations required becomes $t  \sim \frac{ (nm)^{1/(2r+\gamma)}}{1-\sigma_2}$, which corresponds to $t_{\text{Single-Machine}}$ scaled by a certain power of $1/(1-\sigma_2)$ (in our setting the power is $1$). This represents the setting where the choice of step size aligns with the choice in the single-machine case scaled by $(1-\sigma_2)$, and a linear speed-up occurs when $(1-\sigma_2)^{-1} = O(1)$. Since the network error is decreasing in $m$ in our case (due to concentration), larger step sizes can be chosen for $m > n^{2r/\gamma}$. Specifically, the single-machine step size is now scaled by  $[ (1-\sigma_2) ( \frac{m}{n^{2r/\gamma}} )^{1/(2r+\gamma)}] \vee 1$, yielding a linear speed-up when $(1-\sigma_2)^{-1}  = O(  ( \frac{m}{n^{2r/\gamma}} )^{1/(2r+\gamma)} )$, which, as $m$ increases, is a weaker requirement on the network topology over the standard consensus optimisation setting.

\section{General result: error decomposition and implicit regularisation}
\label{sec:Results:Detailed} 
Theorem \ref{Cor:Main} is a corollary of the next result, which explicitly highlights the interplay between statistics and network topology and the implicit regularisation role of the step size and number of iterations.
\begin{theorem}
\label{thm:MainResult}
Let Assumptions \ref{Assumption:Moments}, \ref{Assumption:Source}, \ref{Assumption:Capacity} hold with $r \geq 1/2$. Let $\eta_{s} = \eta s^{-\theta}$ $\forall s\ge 1$ with $\theta \in (0,3/4)$ and $\eta \in (0,\kappa^{-2}]$. If $t/2 \geq  \lceil \frac{ (r+1)\log(t) }{1-\sigma_2}\rceil  =:  t^{\star}$, then for all $v \in V$, $\alpha \in [0,1/2]$ and $\gamma^{\prime} \in [1,\gamma]$:
\begin{align}
&\!\E[\mathcal{E}(\omega_{t+1,v})] 
- \inf_{\omega \in H} \mathcal{E}(\omega) \nonumber \\
&\!\leq 
\Big[ 
q_1(\eta t^{1-\theta} )^{-2r} \!+ q_2 
(nm)^{-2r/(2r+\gamma)}
\Big(1\!\vee\! (nm)^{-2/(2r+\gamma)} (\eta t^{1-\theta} )^2
\!\vee\! t^{-2}(\eta t^{1-\theta})^2 \Big)
\Big] \log^2(t)\! \label{equ:MainTheorem:1}\\
 &\quad+ q_3 \frac{ \log^2(n) \log^2(t^{\star})}{m} 
\Big(  \eta^2 t^{-2r} 
\vee (m^{-1} (\eta t^{\star})^{1 + 2\alpha})  \vee (\eta t^{\star})^{\gamma^{\prime} + 2 \alpha}
\Big) 
\label{equ:MainTheorem:2} \\
& \quad + 
q_4 
\frac{ \log^4(n) \log^2(t) }{m^2} 
\Big( 1 \vee (\eta t^{1-\theta})^{2} \vee t^{-2}(\eta t^{1-\theta})^{4} \Big)
\Big( (m^{-1} \eta t^{1-\theta}) \vee (\eta t^{1-\theta})^{\gamma}\Big)
\label{equ:MainTheorem:3}
\end{align}
where $q_1,q_2,q_3,q_4$ are all constants depending on $\kappa^2,\|\mathcal{T}_{\rho}\|,M,\nu,r,R,\gamma,c_{\gamma}$.
\end{theorem}

The bound in Theorem \ref{thm:MainResult} shows that the excess risk has been decomposed into three main terms, as detailed in Section \ref{sec:ProofSketch:ErrorDecomp}. The first term \eqref{equ:MainTheorem:1} corresponds to the error achieved by Single-Machine Gradient Descent run on all $nm$ samples. It consists of both bias and sample variance terms \citep{lin2017optimal}. The second two terms \eqref{equ:MainTheorem:2} and \eqref{equ:MainTheorem:3} characterise the network error due to the use of a decentralised protocol. These terms decrease with the number of samples $m$ owned by each agent. This captures the fact that, as agents are given samples from the \emph{same} unknown distribution, agents are in fact solving the same learning problem and their local empirical loss functions concentrate to the same objective as $m$ increases. The decentralised error term is itself composed of two terms which decay at different rates with respect to $m$. The  term in \eqref{equ:MainTheorem:2} is dominant and decays at the order of $\widetilde{O}(1/m)$. This can be interpreted as the consensus error seen in the works \cite{Nedic2009,DAW12} for instance. As in that setting, this quantity is also increasing with the step size $\eta $ and decreasing with the spectral gap of the communication matrix $1-\sigma_2$, as encoded by $t^{\star}$. The  term  \eqref{equ:MainTheorem:3} decays at the faster rate of $\widetilde{O}(1/m^2)$. This is a higher-order error term that is not appearing in the error decomposition when the covariance operator $\mathcal{T}_{\rho}$ is assumed to be known to agents. This quantity arises from the interaction between the local averaging on the network through $P$ and what has been previously labelled as the 
``multiplicative'' noise in the single-machine single-pass stochastic gradient setting for least squares \citep{dieuleveut2017harder}, i.e.\ the empirical covariance operator interacting with the iterates at each step. Section \ref{sec:ProofSketch:StatisticalAnalysis} provides a high-level illustration of the analysis of the Network Error terms \eqref{equ:MainTheorem:2} and \eqref{equ:MainTheorem:3}. 

The bound in Theorem \ref{thm:MainResult} shows how the algorithmic parameters---step size and number of iterations---act as regularisation parameters for Distributed Gradient Descent, following what is seen in the single-machine setting. Theorem \ref{Cor:Main} demonstrates how optimal statistical rates can be recovered by tuning these parameters appropriately with respect to the network topology, network size, number of samples, and with respect to the estimation problem itself. 
The bound in Theorem \ref{Cor:Main} is obtained from the bound in Theorem \ref{thm:MainResult} by first tuning the quantity $\eta t$ to the order $(nm)^{1/(2r+\gamma)}$ so that the bias and variance terms in \eqref{equ:MainTheorem:1} achieve the optimal statistical rate. This leaves the tuning of the remaining degree of freedom (say $\eta$) to ensure that also the network error achieves the optimal statistical rate. The high-level idea is the following. As $m$ increases, the network error is dominated by the term in \eqref{equ:MainTheorem:2} that is proportional to the factor $(\eta t^{\star})^{\gamma^{\prime} +2 \alpha}/m$. There are two ways to choose the largest possible step size $\eta$ to guarantee that this factor is $\widetilde{O}((nm)^{-2r/(2r+\gamma)})$, depending on whether the rate of concentration of the batched gradients held by agents is faster than the optimal statistical rate or not, i.e., whether $m \ge n^{2r/\gamma}$ is true or not (cf.\ Section \ref{sec:Results:Main}). 
The two cases yield the factors $( \frac{(nm)^{2r/(2r+\gamma)}}{ m (1-\sigma_2)^{\gamma} })^{1/\gamma} \vee 1$ and $\frac{(nm)^{r/(2r+\gamma)}}{\sqrt{m}(1-\sigma_2)}$ in Theorem \ref{Cor:Main}, corresponding to the choice $\gamma^{\prime} = \gamma \,,\, \alpha = 0$ and $\gamma^{\prime} =1 \,,\, \alpha = 1/2$, respectively. If the concentration of the batched gradients held by agents fully compensates for the network error, i.e. $m\ge\frac{ n^{2r/\gamma}}{(1-\sigma_2)^{2r+\gamma}}$, 
then $(\eta t^{\star})^{\gamma^{\prime} +2 \alpha}/m \simeq (nm)^{-2r/(2r+\gamma)}$ with a constant step size and  $t_{\text{Distributed}} \sim t_{\text{Single-Machine}} \sim (nm)^{1/(2r+\gamma)}$, yielding the regime where a linear speed-up occurs. For more details on the parameters $\alpha,\gamma^{\prime}$, see Lemma \ref{lem:DistributedError:Terma} in Appendix \ref{sec:DecentralisedError:Terma}.

\section{Future directions}
\label{sec:Conclusion}
We highlight some of the features of our contribution and outline directions for future research.
\paragraph{Non-parametric setting} We prove bounds in the attainable case $r \geq 1/2$. The non-attainable case $r < 1/2$ is known to be more challenging \citep{lin2017optimal}, and it is natural to investigate to what extent our results can be extended to that setting. 
We consider the case $\gamma>0$ which does not include the finite-dimensional setting $H = \mathbb{R}^{d}$, $\gamma =0$, where the optimal rate is $O(d/(nm))$ \citep{tsybakov2003optimal}. 
While adapting our results to this setting requires minor modifications, optimal bounds would only hold for ``easy'' estimation problems with $r > 1$ due to the higher-order term in the network error.
Improvements require getting better bounds on this term, potentially using a different learning rate.

\paragraph{General loss functions}
The analysis that we develop is specific to the square loss, which yields the bias/variance error decomposition
and allows to get explicit characterisations by expanding the squares. 
While the concentration phenomena that we exploit are generic, different techniques are required to extend our analysis to other losses, as in the single-machine setting. The statistical proximity of agents' functions in the finite-dimensional setting has been investigated in \cite{2018arXiv180906958R}.

\paragraph{Statistics/communication trade-off with sparse/randomised gossip}
In this work we show that when agents hold sufficiently many samples, then Distributed and Single-Machine Gradient Descent achieve the optimal statistical rate with the same order of iterations. This motivates balancing and trading off communication and statistics, e.g., investigating statistically robust procedures in settings when agents communicate with a subset of neighbours, either deterministically or randomly \citep{boyd2006randomized,dimakis2008geographic,benezit2010order}.

\paragraph{Stochastic gradient descent and mini-batches}
Our work exploits concentration of gradients around their means, so full-batch gradients (i.e.\ batches of size $m$) yield the concentration rate $1/m$. In single-machine learning, stochastic gradient descent \cite{robbins1985stochastic} has been shown to achieve good statistical performance in a variety of settings while allowing for computational savings. Extending our findings to stochastic methods with appropriate mini-batch sizes is another venue for future investigation.

\subsubsection*{Acknowledgments}
Dominic Richards is supported by the EPSRC and MRC through the OxWaSP CDT programme (EP/L016710/1).
Patrick Rebeschini is supported in part by the Alan Turing Institute under the EPSRC grant EP/N510129/1. We would like to thank Francis Bach, Lorenzo Rosasco and Alessandro Rudi for helpful discussions.

\bibliographystyle{plain}
\bibliography{References}

\newpage

\appendix 

\section{Remarks}
\label{sec:remarks}

In this section we present some remarks about our work.

\paragraph{Alternative protocol}
\label{remark:Protocol}
The protocol investigated in \cite{Nedic2009} updates the iterates via\\
$\omega_{t+1,v} 
= 
\sum_{w \in V}
P_{vw} \omega_{t,w} - 
\eta_{t} 
\frac{1}{m} 
\sum_{i=1}^{m}
\big(
\langle \omega_{t,v},x_{i,v} \rangle_H - y_{i,v}
\big)x_{i,v}$. The original motivations for this protocol are that it is fully decentralised, that agents are only required communicate locally, and that it reduces to a distributed averaging consensus protocol when the gradient is zero. The protocol \eqref{alg:DistributedGD} that we consider preserves these properties while making the analysis easier. For a discussion on the difference between the two protocols we refer to \cite{sayed2014adaptive}.

\paragraph{Network error}
The network error terms \eqref{equ:MainTheorem:2} and \eqref{equ:MainTheorem:3} track the error between the distributed protocol and the ideal single-machine protocol. In the case of a complete graph the deviation is zero so the network terms vanish and the convergence rates for Single-Machine Gradient Descent are recovered. Following the literature on decentralised optimisation, we present our final results (cf. Theorem \ref{thm:MainResult}) in terms of the spectral gap, so plugging in the spectral gap of a complete graph in the bound in Theorem \ref{thm:MainResult} does not immediately yield the Single-Machine Gradient Descent result.

\paragraph{Parameter tuning}
The choice of parameters in Theorem \ref{Cor:Main} depends on the quantities $r$ and $\gamma$ that are related to the estimation problem. In practice, these quantities are often unknown. In the single-machine setting, this lack of knowledge is typically addressed via cross-validation \citep{steinwart2008support}. Investigating the design of decentralised cross-validation schemes is outside of the scope of this work and we leave it to future research. However, we highlight that as we consider implicit regularisation strategies and, in particular, early stopping, model complexity can be controlled with iteration time and this yields computational savings for cross-validation compared to methods that required to solve independent problem instances for different choices of parameters.

\paragraph{Accelerated gossip}
Accelerated gossip schemes can also be considered to yield improved dependence on the network topology, depending on the amount of information agents have access to about the communication matrix $P$.
Accelerated gossip can be achieved by replacing the matrix $P$ by a polynomial of appropriate order, e.g. $k$, leading to $\widetilde P := \sum_{\ell=1}^{k} \alpha_{\ell} P^{\ell}$.
The weights $\{\alpha\}_{\ell =1,\dots,K}$ can be tuned to increase the spectral gap i.e.\ $(1-\sigma_2(\widetilde P))^{-1} \leq (1-\sigma_2)^{-1}$. 
We highlight that the algorithm that we consider only needs to have access to the number of nodes $n$ and the second largest eigenvalue in magnitude $\sigma_2$ of the matrix $P$. Within this framework, one can use Chebyshev polynomials to obtain the improved rate $(1-\sigma_2(\widetilde P))^{-1/2}$, and more information on the spectrum of $P$ yields better rates on the transitive phase \citep{cao2006accelerated,berthier2018gossip}.

\paragraph{Additional requirements in Theorem \ref{thm:MainResult}}
Theorem \ref{thm:MainResult} includes two additional requirements over single-machine gradient descent, which we briefly explain the origins of. The requirement $\theta \leq 3/4$ is purely cosmetic and serves to yield a cleaner bound. For more details, see the proof of Lemma \ref{lem:HighProbBoundTermb} in Section \ref{sec:DecentralisedError:Termb}. The requirement $t/2 \geq \frac{ (r + 1)\log(t)}{1-\sigma_2}$, on the other hand, often arises when analysing Distributed Gradient Descent, see \cite{DAW12}  for instance. In particular, it ensures sufficient iterations have been performed to reach the mixing time of the Markov chain associated to $P$. See Section \ref{sec:DecentralisedError:Terma}.

\paragraph{Communication model}
We include additional details on the communication model. Consider a lockstep communication model where each round lasts for $\tau$ units of time. Within each round, agents send/receive the messages to/from their neighbours in order to implement a single update of algorithm (3). With a gradient evaluation costing 1 unit of time, each iteration of Distributed Gradient Descent takes the following amount of time $m + \tau + \mathrm{Deg}(P)$: $m$ gradient evaluations; $\tau$ in communication delay; $\mathrm{Deg}(P)$ for each agent to aggregating their neighbours and own gradients, as the sum in algorithm (3) $\sum_{w \in V} P_{vw}$ has computational cost $O(\mathrm{Deg}(P))$.  The delay $\tau$ can depend on factors arising from: noisy transmission, compressing or decompressing messages and synchronizing with neighbours. One particular model for $\tau$ is studied within \cite{tsianos2012communication} and discussed in the following remark.  

\paragraph{Comparison to speed-up and communication model within \cite{tsianos2012communication} }
The work \cite{tsianos2012communication} assumes the delay $\tau$ is a linear function of the network degree and some transmit time $T_{\mathrm{Transmit}} \geq 0$ so $\tau = T_{\mathrm{Transmit}} \mathrm{Deg}(P)$. In our work, for sufficiently many samples $m$, the speed-up under this model for any network topology is of the order $\frac{nm}{m + \text{Deg}(P) T_{\mathrm{Transmit}}}$. Meanwhile, the speed-up seen within \cite{tsianos2012communication} is\footnote{
The units of time within \cite[Section 3.2]{tsianos2012communication} are in terms of the time taken to compute a gradient for $nm$ samples, and as such, can be translated into units per gradient computation by multiplying by $nm$. 
} of the order $\frac{nm}{m + \text{Deg}(P) T_{\mathrm{Transmit}}}(1-\sigma_2) $, that is, same as ours but scaled by the spectral gap of the communication matrix $P$.

\section{Proof scheme}
\label{sec:ProofSketch}
In this section we illustrate the main scheme for the proof of Theorem \ref{thm:MainResult}, from which Theorem \ref{Cor:Main} follows. Section \ref{sec:ProofSketch:ErrorDecomp} presents the error decomposition into bias, variance, and network terms. Section \ref{sec:ProofSketch:StatisticalAnalysis} presents the sketch of the statistical analysis for these terms, which is given in full in Section \ref{Appendix:Proofs}.

\subsection{Error decomposition}
\label{sec:ProofSketch:ErrorDecomp}
 The error decomposition is based on the introduction of two auxiliary processes used to compare the iterates of Distributed Gradient Descent \eqref{alg:DistributedGD}.

The first auxiliary process represents the iterates generated if agents were to know the marginal distribution $\rho_X$. Initialised at $\mu_1 = 0$, the process is defined as follows for $t \geq 1$:
\begin{align*}
\mu_{t+1} = \mu_{t} - \eta_{t} \int_{X} (\langle \mu_{t}, x \rangle_{H} - f_{\rho}(x))x d\rho_{X}(x).
\end{align*}
This device has already been used in the analysis of non-parametric regression in the single-machine setting \citep{lin2017optimal}.

The second auxiliary process represents the iterates generated if agents were to be part of a complete graph topology and were to use the protocol given by $P = \frac{1}{n} \mathbf{1} \mathbf{1}^{\top} $. Initialised at $\xi_{1,v} =0$ for all $v \in V$, the process is defined as follows for $t \geq 1$:
\begin{align*}
\xi_{t+1,v} 
= 
\sum_{w \in V} 
\frac{1}{n} 
\bigg(
\xi_{t,w} - \eta_{t} \frac{1}{m} \sum_{i=1}^{m} \!( 
\langle \xi_{t,w}, x_{i,w}\rangle_{H} - y_{i,w}
\big)x_{i,w}
\bigg).
\end{align*}
The analysis of iterative decentralised algorithms typically builds upon the introduction of a device analogous to this one \citep{Nedic2009, DAW12}. 
Initialised at $\xi_1=0$, Single-Machine Gradient Descent is defined as follows for $t \geq 1$:
\begin{align*}
\xi_{t+1} = \xi_t -  \eta_{t} \frac{1}{nm} \sum_{w \in V} \sum_{i=1}^{m} \big( 
\langle \xi_{t}, x_{i,w}\rangle_H - y_{i,w}
\big)x_{i,w}.
\end{align*}
It is easy to see that we have $\xi_{t,v} = \xi_{t}$ for $t \geq 1$ and $v \in V$. This allows us to produce an analysis of Distributed Gradient Descent that relies upon known results for Single-Machine Gradient Descent. 

Let us introduce the linear map $\mathcal{S}_{\rho}: H \rightarrow L^2(H,\rho_{X})$ defined by $\mathcal{S}_{\rho} \omega = \langle \omega,\,\cdot\,\rangle_{H}$. The following error decomposition holds. 
\begin{proposition}
\label{Prop:ErrorDecomp}
For any $t\ge 1$ and $v\in V$ we have 
\begin{align*}
\mathcal{E}(\omega_{t,v}) - \inf_{\omega\in H} \mathcal{E}(\omega)
\leq 
2 \underbrace{\|\mathcal{S}_{\rho} \mu_{t} - f_{H}\|^2_{\rho}}_{(\textbf{Bias})^2}
+ 
4\underbrace{\|\mathcal{S}_{\rho}(\xi_{t} - \mu_{t})\|^2_{\rho}}_{\textbf{Sample Variance}}
+ 
4\underbrace{\|\mathcal{S}_{\rho}(\omega_{t,v} - \xi_{t,v})\|^2_{\rho}}_{\textbf{Network Error}}.
\end{align*}
\end{proposition}
\begin{proof}
From the work in \cite{rosasco2015learning}, 
$
\mathcal{E}(\omega) - \inf_{\omega\in H}\mathcal{E}(\omega)
= 
\|\mathcal{S}_{\rho} \omega- f_{H} \|^2_{\rho}
$
for any $\omega \in H$. Adding and subtracting $\mathcal{S}_{\rho}\mu_{t}$ and using $\|x-y\|^2_{\rho} \leq (\|x\|_{\rho} + \|y\|_{\rho})^2 \leq 2\|x\|^2_{\rho} + 2\|y\|^2_{\rho}$ we get
\begin{align*}
\mathcal{E}(\omega_{t,v}) - \inf_{\omega\in H} \mathcal{E}(\omega)  
& = 
\|\mathcal{S}_{\rho} \omega_{t,v} - \mathcal{S}_{\rho} \mu_{t} 
+ \mathcal{S}_{\rho} \mu_{t} - f_{H} \|_\rho^2
\leq 
2 \|\mathcal{S}_{\rho} \omega_{t,v} - \mathcal{S}_{\rho} \mu_{t} \|^2_{\rho} 
+ 
2 \|\mathcal{S}_{\rho} \mu_{t} - f_{H} \|_\rho^2.
\end{align*}
Following the same steps, adding and subtracting $\mathcal{S}_{\rho}\xi_{t,v}$, we find
\begin{align*}
 \|\mathcal{S}_{\rho} \omega_{t,v} \!-\! \mathcal{S}_{\rho} \mu_{t} \|^2_{\rho}
 & = 
  \|\mathcal{S}_{\rho} \omega_{t,v} \!-\! \mathcal{S}_{\rho} \xi_{t,v} 
  +\mathcal{S}_{\rho} \xi_{t,v}  \!-\! \mathcal{S}_{\rho} \mu_{t} \|^2_{\rho}
    \le 
    2\|\mathcal{S}_{\rho} (\omega_{t,v}  -\xi_{t,v}) \|_{\rho}^{2} 
    + 
    2\|\mathcal{S}_{\rho} (\xi_{t}   - \mu_{t}) \|^2_{\rho}
\end{align*}
where we used the equality of $\{\xi_{s,v}\}_{s\ge 1}$ and $\{\xi_{s}\}_{s\ge 1}$. 
\end{proof}
Proposition \ref{Prop:ErrorDecomp} decomposes the error into three terms. The first term $\|\mathcal{S}_{\rho} \mu_{t} - f_{H}\|_{\rho}^2$ is deterministic and corresponds to the square of the \textbf{Bias} in the single-machine setting \citep{yao2007early}. The second term $\|\mathcal{S}_{\rho}(\xi_{t} - \mu_{t})\|^2_{\rho}$ aligns with what is called the \textbf{Sample Variance} in the single-machine setting, and in this case matches the sample variance obtained for Single-Machine Gradient Descent run on all $nm$ observations. The third term $\|\mathcal{S}_{\rho}(\omega_{t,v} - \xi_{t,v})\|_{\rho}^2$ accounts for the error due to performing a decentralised protocol and we call it the \textbf{Network Error}.

\subsection{Statistical analysis of error terms}
\label{sec:ProofSketch:StatisticalAnalysis}
In this section we illustrate the main ideas of the statistical analysis used to control the error terms in Proposition \ref{Prop:ErrorDecomp}. 
Full details are given in Section \ref{Appendix:Proofs}.
\paragraph{Notation} Let $t$ and $k$ be positive natural numbers with $t -1 \geq k \geq 1$. For any operator $\mathcal{L} : H \rightarrow H$, define $\Pi_{t:k+1}(\mathcal{L}) := (I - \eta_t \mathcal{L}) (I - \eta_{t-1} \mathcal{L})\cdots(I - \eta_{k+1}\mathcal{L})$, with the convention $\Pi_{t:t+1}(\mathcal{L}) := I$, where $I$ is the identity operator on $H$. Let $w_{t:k+1}\equiv w_{t}w_{t-1}\dots w_{k+1} := (w_{t},w_{t-1},\dots,w_{k+1}) \in V^{t-k}$ denote a sequence of nodes in $V$. For a family of operators indexed by the nodes on the graph $\{\mathcal{L}_{v}\}_{v \in V}$, define $\mathcal{L}_{w_{t:k+1}} := (\mathcal{L}_{w_t},\ldots,\mathcal{L}_{w_{k+1}})$ and $\Pi_{t:k+1}(\mathcal{L}_{w_{t:k+1}}) := (I - \eta_{t} \mathcal{L}_{w_{t}})(I - \eta_{t-1}\mathcal{L}_{w_{t-1}})\cdots(I - \eta_{k+1}\mathcal{L}_{w_{k+1}})$, with  $\Pi_{t:t+1}(\mathcal{L}_{w_{t:t+1}}) := I$. 
Let $P_{w_{t:k+1}} := P_{w_{t}w_{t-1}}P_{w_{t-1}w_{t-2}} \cdots \\ P_{w_{k+2}w_{k+1}}$ be the probability of the path generated by a Markov Chain with transition kernel $P$. 
For each agent $v \in V$, let $\mathcal{T}_{\mathbf{x}_{v}}:H \rightarrow H$ with $\mathcal{T}_{\mathbf{x}_{v}} = \frac{1}{m} \sum_{i=1}^{m} \langle \,\cdot\,, x_{i,v} \rangle_{H} x_{i,v}$ be the empirical covariance operator associated to the agent's own data $\mathbf{x}_{v}$, and let $\mathcal{T}_{\mathbf{x}_{w_{t:k+1}}}:= (\mathcal{T}_{\mathbf{x}_{w_{t}}},\dots, \mathcal{T}_{\mathbf{x}_{w_{k+1}}})$.
For $k \geq 1, v \in V$, let $N_{k,v} \in H$ be a random variable that only depends on the randomness in $\mathbf{z}_{v}$ and that has zero mean, $\E[N_{k,v}] = 0$.
The random variable $N_{k,v}$, formally defined in \eqref{Equ:LocalErrorTerm} in Section \ref{Sec:Proof:Distributed}, captures the sampling error introduced at iteration $k$ of gradient descent by agent $v$. For the discussion below it suffices to mentioned the two above properties.

The following paragraphs discuss the analysis for each of the error terms.
\paragraph{Bias}
The analysis follows the single-machine setting  and is given in Proposition \ref{Prop:BiasBound} in Section \ref{Sec:Proof:Bias}.
\paragraph{Sample Variance}
The analysis follows the single-machine setting \citep{lin2017optimal}, although the original result yields a high probability bound with a requirement on the number of samples $nm$. We therefore follow the result in \cite{lin2018optimal} which yields a bound in high probability  without a condition on the sample size. The bound for this term is presented in Theorem \ref{thm:SampleVarianceBound} in Section \ref{Sec:Proof:Variance}.
\paragraph{Network Error}
Unraveling the iterates (Lemma \ref{Lem:DistributedError:Decomposed} in Section \ref{Sec:Proof:Distributed}) we get, for any $v \in V, t \geq 1$:
\begin{align*}
 \|\mathcal{S}_{\rho}(\omega_{t+1,v} \!-\! \xi_{t+1,v}) \|_{\rho}
 = 
	\bigg\| \sum_{k=1}^{t} \eta_{k} \!\!\sum_{w_{t:k} \in V^{t-k+1}} 
	\!\!\bigg( 
	P_{v w_{t:k}} \!-\! \frac{1}{n^{t-k+1}}  \bigg)
	\mathcal{T}^{1/2}_{\rho}
	\Pi_{t:k+1}(\mathcal{T}_{\mathbf{x}_{w_{t:k+1}}}) N_{k,w_{k}}
	\bigg \|_{H}.
\end{align*}
This characterisation makes explicit the dependence of the network error on both the communication protocol used by the agents, via the dependence on the mixing properties of the gossip matrix $P$ along each path $v w_{t:k}$, and on the statistical properties of the problem, via the product of empirical covariance operators held by the agents along each path $w_{t:k+1}$.
As the randomness in the quantities $N_{k,w_{k}}$ might depend on the randomness in the empirical covariance operators, we further decompose the network error into two terms so that we can use the property $\E[N_{k,w_{k}}]=0$. By adding and subtracting the terms $\Pi_{t:k+1}(\mathcal{T}_{\rho})$ inside the sums we have
\begin{align*}
& \| \mathcal{S}_{\rho} ( \omega_{t+1,v}  - \xi_{t+1,v}) \|_{\rho}^2 \leq 
2 
\underbrace{ 
	\bigg\| \sum_{k=1}^{t} \eta_{k} \sum_{w_{t:k} \in V^{t-k+1}} 
	\bigg( 
	P_{v w_{t:k}} - \frac{1}{n^{t-k+1}}  \bigg)
	\mathcal{T}^{1/2}_{\rho}\Pi_{t:k+1}(\mathcal{T}_{\rho}) N_{k,w_{k}}
	\bigg\|_{H}^2 
	}_{
		(\textbf{Population Covariance Error})^2
	}\\
&\quad+ 
2 
\underbrace{
	\bigg\| \sum_{k=1}^{t} \eta_{k} \sum_{w_{t:k} \in V^{t-k+1}} 
	\bigg( 
	P_{v w_{t:k}} - \frac{1}{n^{t-k+1}}  \bigg)
	\mathcal{T}^{1/2}_{\rho}
	\big( 
	\Pi_{t:k+1}(\mathcal{T}_{\mathbf{x}_{w_{t:k+1}}})
	- \Pi_{t:k+1}(\mathcal{T}_{\rho})
	) N_{k,w_{k}}
	\bigg \|_{H}^2
	}_{ (\textbf{Residual Empirical Covariance Error})^2
	}.
\end{align*}
From a statistical point of view, the \textbf{Population Covariance Error} term only depends on the population covariance via the quantities $\Pi_{t:k+1}(\mathcal{T}_{\rho})$, and the only source of randomness is given by $N_{k,w_k}$. Using concentration for $N_{k,w_k}$, the square of this error term can be bounded by a quantity that decreases as $\widetilde O (1/m)$, as announced in Section \ref{sec:Results:Detailed} alongside the discussion of Theorem \ref{thm:MainResult}. On the other hand, the \textbf{Residual Empirical Covariance Error} term depends on \emph{deviations} between the empirical covariance and the population covariance via the quantities $\Pi_{t:k+1}(\mathcal{T}_{\mathbf{x}_{w_{t:k+1}}})
	- \Pi_{t:k+1}(\mathcal{T}_{\rho})$. Exploiting the additional concentration of these factors allows us to bound the square of this error term by a higher-order quantity that decreases as $\widetilde O (1/m^2)$.
	
	We now present a separate discussion on the analysis for these two error terms, emphasizing the interplay between network topology (mixing of random walks on graphs) and statistics (concentration). 
	The final bound for the network error is presented in Theorem \ref{thm:DecentralisedErrorBound} in Section \ref{Sec:Proof:Distributed}. 

\paragraph{Population Covariance Error}
Expanding the square yields a summation over all pairs of paths:
$$\bigg\|\sum_{k=1}^{t} \sum_{w_{t:k}  \in V^{t-k+1}}  a_{k,w_{t:k}}\bigg\|^2_H = \sum_{k,k^\prime =1}^{t} \sum_{w_{t:k} \in V^{t-k+1}} \sum_{w_{t:k^{\prime}}^{\prime} \in V^{t-k^{\prime}+1}} \langle a_{k,w_{t:k}} a_{k^{\prime},w^{\prime}_{t:k^{\prime}}} \rangle_{H}$$ for properly defined quantities  $a_{k,w_{t:k}}$ (the dependence on $v$ is neglected). When taking the expectation, as the random variables $\{N_{k,v}\}_{k \geq 1, v \in V}$ have zero mean and are independent across agents $v \in V$, the only paths left are those that intersect at the final node, i.e.\ $w_{t:k},w^{\prime}_{t:k^{\prime}}$ such that $w_{k} = w_{k^{\prime}}$. Moreover, as all agents have identically distributed data, the remaining expectation no longer depends on the final node of the paths. The remaining quantity is then analysed by bounding the probability of the two paths intersecting at the final node in terms of the second largest eigenvalue in magnitude of $P$ and by bounding the inner product by the norm product. This yields
\begin{align*}
    \E[(\textbf{Pop. Cov. Error})^2] \leq 
    \E \bigg[ \bigg( 
\sum_{k=1}^{t} \sigma_2^{t-k+1} \eta_k 
\|\mathcal{T}^{1/2}_{\rho} \Pi_{t:k+1}(\mathcal{T}_{\rho}) N_{k,v}\|_{H}
\bigg)^2 \bigg].
\end{align*}

Denoting the mixing time associated to  $P$ as $t^{\star}$,  the series is divided into well-mixed and poorly-mixed terms, respectively, $k \leq t - t^{\star}$ and 
 $k \geq t-t^{\star}$. The well-mixed terms are controlled by $\sigma_2^{t-k+1}$. Meanwhile, for the poorly-mixed terms begin by taking for $\lambda > 0 $ $\max_{k=1,\dots,t} \big\{  \|(\mathcal{T}_{\rho}+\lambda I )^{-1/2} N_{k,v}\|_{H}^2  \big\}$ outside of the series. The expectation of this maximum is controlled through concentration and becomes 
 $\widetilde{O}( \frac{1}{m^2 \lambda} + \frac{1}{m \lambda^{\gamma^{\prime}}})$ for $\gamma^{\prime} \in [1,\gamma]$. The remaining  series is controlled through the contraction of the term $\|\mathcal{T}_{\rho}^{1/2} \Pi_{t:k+1}(\mathcal{T}_{\rho}) (\mathcal{T}_{\rho} + \lambda I )^{1/2}\|$ and  choosing $\lambda  \simeq 1/(\eta t^{\star}) $. These two steps lead to this term being of the order $O( \frac{ \eta t^{\star}}{m^2} + \frac{ (\eta t^{\star})^{\gamma^{\prime}} }{m} )$, which dominates the well-mixed terms and contributes to the dependence on the inverse of the spectral gap of $P$. The free parameter $\gamma^{\prime} \in [1,\gamma]$ is left open as a smaller step size $\eta$ is used to control this term when $m \leq n^{2r/\gamma}$. The final bound is given in Lemma  \ref{lem:DistributedError:Terma} in Section \ref{sec:DecentralisedError:Terma}.

\paragraph{Residual Empirical Covariance Error}
The analysis of this term is based on the following identity (Proposition \ref{prop:QuadraticExpansion} in Section \ref{sec:DecentralisedError:Termb}), for any $t -1 \geq k$ and any $w_{t:k+1} \in V^{t-k}$:
\begin{align*}
& \Pi_{t:k+1}(\mathcal{T}_{\mathbf{x}_{w_{t:k+1}}}) 
-  \Pi_{t:k+1}(\mathcal{T}_{\rho}) 
= 
\sum_{j=k+1}^{t} \eta_{j} \Pi_{t:j+1}(
\mathcal{T}_{\rho})(\mathcal{T}_{\rho} - \mathcal{T}_{\mathbf{x}_{w_{j}}}) \Pi_{j-1:k+1}(\mathcal{T}_{\mathbf{x}_{w_{j-1:k+1}}}).
\end{align*}
The above decomposition has two key properties. Firstly, it depends upon differences between the empirical covariance operators $\mathcal{T}_{\mathbf{x}_{w_{j}}}$ and its expectation $\mathcal{T}_{\rho}$. This allows concentration to be used, and, alongside the concentration for $N_{k,v}$, it ensures that $(\textbf{Resid. Emp. Cov. Error})^2$ is of order $\widetilde{O}(1/m^2)$. Secondly, it is of the form $ \sum_{j=k+1}^{t} \eta_{j}  \Pi_{t:j+1}(\mathcal{T}_{\rho})[\cdots]$, where $[\cdots]$ indicates the right most factors and the quantity shown aligns with the filter function for gradient descent \citep[Example 2]{lin2018optimal}.
Once again the contractive property of the quantity $\Pi_{t:j+1}(\mathcal{T}_{\rho})$ allows to give sharper rates with respect to the step size and number of iterations.
Without it, the choice of step size $\eta_{t} = \eta t^{-\theta}$ would yield a bound for $(\textbf{Resid. Emp. Cov. Error})^{2}$ of the order $ \big( \sum_{k=1}^{t} \eta_{k} \sum_{j=k+1}^{t-1} \eta_{j} \big)^2 \simeq (\eta t^{1-\theta})^{4}$. The contraction allows to show that $(\textbf{Resid. Emp. Cov. Error})^{2}$ grows at the reduced order $(\eta t^{1-\theta})^{3}$, and the addition of the capacity assumption allows it to be further reduced to the order $(\eta t^{1-\theta})^{2+\gamma}$. The final high-probability bound is given in Lemma \ref{lem:HighProbBoundTermb} in Section \ref{sec:DecentralisedError:Termb}. This being stronger than the bound in expectation required for Theorem \ref{thm:MainResult}.

\section{Proofs}
\label{Appendix:Proofs}
Before going on to present proofs for the main result some notation  is introduced following \cite{rosasco2015learning,lin2017optimal}. Some notation is repeated from the previous sections, as additional details are included. Adopt the convention for sums $\sum_{k=t+1}^{t} = 0$. 
For a given bounded operator $\mathcal{L}: L^2(H,\rho_{X}) \rightarrow H$, let $\|\mathcal{L}\|$ denote the operator norm of $\mathcal{L}$, i.e.\ $\|\mathcal{L}\| = \sup_{f \in L^2(H,\rho_{X}),\|f\|_{\rho} =1} \|\mathcal{L}f\|_{H}$. 
Let $\mathcal{S}_{\rho} : H \rightarrow L^2(H,\rho_{X})$ be the linear map $\omega \rightarrow \langle \omega,\,\cdot\,\rangle_{H}$,which is bounded by $\kappa$ under Assumption \ref{Assumption:BoundedProduct}. Consider the adjoint operator $\mathcal{S}^{\star}_{\rho}: L^{2}(H,\rho_{X}) \rightarrow H$, the covariance operator $\mathcal{T}_{\rho}: H \rightarrow H$ given by $\mathcal{T}_{\rho} = \mathcal{S}^{\star}_{\rho} \mathcal{S}_{\rho}$, and the operator $\mathcal{L}_{\rho}: L^{2}(H,\rho_{X}) \rightarrow L^2(H,\rho_{X})$ given by $\mathcal{L}_{\rho} = \mathcal{S}_{\rho} \mathcal{S}^{\star}_{\rho}$. We have $\mathcal{S}^{\star}_{\rho} g = \int_{X} x g(x) d \rho_{X}(x)$ and $\mathcal{T}_{\rho} = \int_{X} \langle \,\cdot\,,x\rangle_{H} x d\rho_{X}(x)$. For any $\omega \in H$ the following isometry property holds \citep{steinwart2008support} 
\begin{align*}
\|\mathcal{S}_{\rho} \omega \|_{\rho} 
= 
\|\sqrt{\mathcal{T}_{\rho}} \omega \|_{H}.
\end{align*}
The following notation was utilised in the analysis of Single-Machine Gradient Descent \citep{rosasco2015learning,lin2017optimal}. In this case it aligns with all of the observations in the network $\mathbf{y}:= \{ y_{i,v}\}_{i=1,\dots,m \,, v \in V} \in \mathbb{R}^{m|V|}$ and $\mathbf{x} = \{x_{i,v}\}_{i=1,\dots,m \,, v \in V}$.
Define the sampling operator $\mathcal{S}_{\mathbf{x}}: H \rightarrow \mathbb{R}^{m|V|}$  by $\big( \mathcal{S}_{\mathbf{x}}\omega\big)_{(i,v)} = \langle \omega, x_{i,v} \rangle_{H}$, for  $i=1,\dots,m, v \in V$. Let $\|\cdot\|_{\mathbb{R}^{m|V|}}$ denote the Euclidean norm in in $\mathbb{R}^{m|V|}$ times the factor $1/\sqrt{nm}$. Its adjoint  operator $\mathcal{S}^{\star}_{\mathbf{x}}:\mathbb{R}^{m |V|} \rightarrow H$, defined by $\langle \mathcal{S}^{\star}_{\mathbf{x}} \mathbf{y},\omega \rangle_{H} = \langle \mathbf{y},\mathcal{S}_{\mathbf{x}} \omega \rangle_{\mathbb{R}^{m |V|}}$ for \\ 
$\mathbf{y} \in \mathbb{R}^{m|V|}$, is given by $\mathcal{S}^{\star}_{\mathbf{x}} \mathbf{y} = \frac{1}{nm} \sum_{v \in V} \sum_{i=1}^{m} y_{i,v} x_{i,v}$.  Define the covariance operator with respect to all of the samples $\mathcal{T}_{\mathbf{x}} : H \rightarrow H $ such that $\mathcal{T}_{\mathbf{x}} = \mathcal{S}_{\mathbf{x}}^{\star} \mathcal{S}_{\mathbf{x}}$. We have 
\begin{align*}
\mathcal{T}_{\mathbf{x}} = \frac{1}{nm} \sum_{v \in V} \sum_{i=1}^{m} \langle \,\cdot\,, x_{i,v}\rangle_{H} x_{i,v}.
\end{align*}
The following notation  is analogous to the single-machine notation just introduced, although now with respect to the datasets held by individual agents, i.e.\  $\mathbf{x}_{v}$ and $\mathbf{y}_{v}$ for $v \in V$. Let $\mathcal{S}_{\mathbf{x}_{v}}: H \rightarrow \mathbb{R}^{m}$ with 
$(\mathcal{S}_{\mathbf{x}_{v}} \omega)_i = \langle \omega, x_{i,v}\rangle_{H}$ for $i=1,\dots,m$. Let $\|\cdot\|_{\mathbb{R}^{m}}$ be the Euclidean norm in $\|\cdot\|_{\mathbb{R}^{m}}$ times $1/\sqrt{m}$. Its  adjoint operator $\mathcal{S}^{\star}_{\mathbf{x}_{v}}:\mathbb{R}^{m} \rightarrow H$, defined by $\langle \mathcal{S}^{\star}_{\mathbf{x}_{v}} \mathbf{y}_{v},\omega \rangle_{H} = \langle \mathbf{y}_{v},\mathcal{S}_{\mathbf{x}_{v}}\omega \rangle_{\mathbb{R}^{m}}$ for $\mathbf{y}_{v} \in \mathbb{R}^{m}$, is given by $\mathcal{S}^{\star}_{\mathbf{x}_{v}} \mathbf{y}_{v} = \frac{1}{m} \sum_{i=1}^{m} y_{i,v} x_{i,v}$. The  empirical covariance operator $\mathcal{T}_{\mathbf{x}_{v}}: H \rightarrow H$ is such that $\mathcal{T}_{\mathbf{x}_{v}} = \mathcal{S}^{\star}_{\mathbf{x}_{v}} \mathcal{S}_{\mathbf{x}_{v}}$, with $\mathcal{T}_{\mathbf{x}_{v}} = \frac{1}{m} \sum_{i=1}^{m} \langle \,\cdot\,, x_{i,v}\rangle_{H} x_{i,v}$. 

Using this notation, the processes $\{\mu_{t}\}_{t \geq 1}$, $\{\omega_{t,v}\}_{t \geq 1}$, and $\{\xi_{t} \}_{t \geq 1}$ can be rewritten as follows.\\The population process reads
\begin{align*}
\mu_{t+1} = \mu_{t} - \eta_{t} \big( 
\mathcal{T}_{\rho} \mu_{t} - \mathcal{S}^{\star}_{\rho} f_{\rho}
\big).
\end{align*}
The gossiped process reads
\begin{align*}
\omega_{t+1,v} 
= 
\sum_{w \in V}
P_{vw}
\Big(
\omega_{t,w} 
- 
\eta_{t}
\big(
\mathcal{T}_{\mathbf{x}_{w}} \omega_{t,w} - \mathcal{S}^{\star}_{\mathbf{x}_{w}} \mathbf{y}_{w}
\big)
\Big).
\end{align*}
The single-machine process reads
\begin{align*}
\xi_{t+1}
= 
\xi_{t} 
- \eta_{t} 
\big(
\mathcal{T}_{\mathbf{x}} \xi_{t}
- 
\mathcal{S}^{\star}_{\mathbf{x}} \mathbf{y}
\big). 
\end{align*}
The next three sections present bounds for the three error terms introduced in Proposition \ref{Prop:ErrorDecomp}. Section  \ref{Sec:Proof:Bias} presents a bound for the Bias term, which follows directly from the results in \cite{lin2017optimal} and references therein. Section \ref{Sec:Proof:Variance} establishes a bound for the Sample Variance term, which follows from results in \cite{lin2018optimal}. Section \ref{Sec:Proof:Distributed} develops bounds for the Network Error term, which are a novel contribution of this work. Section \ref{sec:Appendix:ConstructingFinalBound} brings the results of the previous three sections together to establish the proofs of Theorem \ref{thm:MainResult} and Theorem \ref{Cor:Main}, respectively. Section \ref{sec:UsefulIneq} includes useful inequalities that are needed to establish our results.

\subsection{Bias}
\label{Sec:Proof:Bias}
The following bound on the Bias term $\|\mathcal{S}_{\rho} \mu_{t} - f_{H}\|^2_{\rho}$ is taken from \cite{lin2017optimal}, inspired by \cite{yao2007early,rosasco2015learning}. 
\begin{proposition}{\cite[Appendix C Proposition 2]{lin2017optimal}}
\label{Prop:BiasBound}
Under Assumption \ref{Assumption:Source}, let $\eta \kappa^2 \leq 1$. Then for any $t \in \mathbb{N}$,
\begin{align*}
\|\mathcal{S}_{\rho} \mu_{t} - f_{H}\|_{\rho}
\leq 
R \bigg( 
\frac{r}{2 \sum_{j=1}^{t} \eta_{j}}
\bigg)^{r}.
\end{align*}
In particular, if $\eta_{t} = \eta t^{-\theta}$ for all $t \in \mathbb{N}$, with $\eta \in (0,\kappa^{-2}]$ and $\theta \in [0,1)$ then 
\begin{align*}
\|\mathcal{S}_{\rho} \mu_{t} - f_{H}\|_{\rho}
 & \leq 
R r^r \eta^{-r} t^{r(\theta-1)}.
\end{align*}
\end{proposition}

\subsection{Sample Variance}
\label{Sec:Proof:Variance}
In this section we establish a bound for the expectation of the Sample Variance term 
$\E[ \|\mathcal{S}_{\rho}(\xi_{t} - \mu_{t})\|^2_{\rho}]$.
The following lemma summaries a number of intermediary steps in  \cite{lin2017optimal} for bounding the Sample Variance term. It arises from representing the iterates $\{\xi_t - \mu_t\}_{t \geq 1}$ in terms of the stochastic sequence $\{N_{k}\}_{k \geq 1}$ which characterises the sample noise introduced in the iterations of gradient descent. These terms are controlled via the empirical covariance operator $\mathcal{T}_{\mathbf{x}}$ and the population covariance operator $\mathcal{T}_{\rho}$ while introducing the pseudo-regularisation parameter $\lambda > 0$ and utilising the contractive property of the gradient updates.
For the following, let us introduce the notation $\mathcal{T}_{\rho,\lambda} = \mathcal{T}_{\rho} + \lambda I$ and $\mathcal{T}_{\mathbf{x},\lambda} = \mathcal{T}_{\mathbf{x}} + \lambda I$.
\begin{lemma}
\label{Lem:SampleVariance:Bound}
Let $\eta_{1} \kappa^2 \leq 1$ and $0 \leq \lambda $. For any $t\in \mathbb{N}$ we have
\begin{align*}
& \|\mathcal{S}_{\rho}(\xi_{t+1} - \mu_{t+1})\|_{\rho}\\
& \leq 
\bigg( 
\sum_{k=1}^{t-1} \frac{\eta_{k} \|\mathcal{T}^{-1/2}_{\rho,\lambda} N_{k}\|_{H}  }{2\sum_{i=k+1}^{t} \eta_i } 
+ \lambda \sum_{k=1}^{t-1} \eta_{k} \|\mathcal{T}^{-1/2}_{\rho,\lambda} N_{k}\|_{H} + 
\|\mathcal{T}_{\rho}\|^{1/2}(\|\mathcal{T}_{\rho}\| + \lambda)^{1/2} \eta_t \|\mathcal{T}^{-1/2}_{\rho,\lambda} N_{t}\|_{H}
\bigg) \\
&\quad \ \times \|\mathcal{T}_{\mathbf{x},\lambda}^{-1/2} \mathcal{T}^{1/2}_{\rho} \|
\|\mathcal{T}_{\mathbf{x},\lambda}^{-1/2} \mathcal{T}^{1/2}_{\rho,\lambda} \|,
\end{align*}
where 
\begin{align}
\label{equ:NTerm}
N_{k} = (\mathcal{T}_{\rho} \mu_{k} - \mathcal{S}^{\star}_{\rho} f_{\rho}) 
- (\mathcal{T}_{\mathbf{x}} \mu_{k} - \mathcal{S}^{\star}_{\mathbf{x}} \mathbf{y}), \quad \forall k \in \mathbb{N}. 
\end{align}
\end{lemma}
\begin{proof}
The proof of this result follows the proof of \cite[Proposition 3]{lin2017optimal}.
\end{proof}
The two quantities left to control are $\|\mathcal{T}^{-1/2}_{\rho,\lambda} N_{k}\|_{H}$ for $k\in \mathbb{N}$ as well as $\|(\mathcal{T}_{\mathbf{x}} + \lambda I)^{-1/2} \mathcal{T}^{1/2}_{\rho} \|^2$. The first of these quantities is controlled by \cite[Lemma 18]{lin2017optimal} which is summarised in the following lemma.
\begin{lemma}{\cite[Lemma 18]{lin2017optimal}}
\label{Lem:SampleVariance:NBound}
Let Assumptions \ref{Assumption:Moments}, \ref{Assumption:Source}, \ref{Assumption:Capacity} hold with $r \geq 1/2$ and  $\{N_{k}\}_{k \geq 1}$ be as in \eqref{equ:NTerm}. For any $\lambda> 0$, 
with probability at least $1-\delta$, the following holds $\forall k \in \mathbb{N}$
\begin{align*}
\|(\mathcal{T}_{\rho} + \lambda I )^{-1/2} N_{k} \|_{H} 
\leq 
4 (R \kappa^{2r} + \sqrt{M} ) 
\bigg( 
\frac{\kappa}{n m \sqrt{\lambda}} + \frac{\sqrt{2 \sqrt{\nu} c_{\gamma}}}{\sqrt{n m \lambda^{\gamma}}} \bigg) \log\frac{4}{\delta}.
\end{align*}
\end{lemma}
The next lemma from \cite[Lemma 19 Remark 1]{lin2018optimal} controls $\|(\mathcal{T}_{\mathbf{x}} + \lambda I)^{-1/2} \mathcal{T}^{1/2}_{\rho} \|^2$. 
\begin{lemma}{\citep[Lemma 19, Remark 1]{lin2018optimal}}
\label{Lem:Regularised:Operator:Bound}
Let $\delta  \in (0,1)$ and $\lambda  = (nm)^{-p}$ for some $p \geq 0$. With probability at least $1-\delta$ the following holds 
\begin{align*}
& \| \mathcal{T}_{\rho}^{1/2}(\mathcal{T}_{\mathbf{x}} + \lambda )^{-1/2}\|^2
 \leq 
\| (\mathcal{T}_{\rho} + \lambda I )^{1/2}(\mathcal{T}_{\mathbf{x}} + \lambda )^{-1/2}\|^2\\
& \leq 
24 \kappa^2 
\bigg( 
\log \frac{4 \kappa^2 (c_{\gamma} + 1)}{\delta \|\mathcal{T}_{\rho}\|}
+ 
p \gamma \min \bigg( \frac{1}{e(1-p)_{+}}, \log nm \bigg) \bigg)
(1 \vee (nm)^{ p -1} ).
\end{align*}
\end{lemma}
Bringing together the three previous results yields the following high-probability bound for the Sample Variance term.
\begin{proposition}
\label{prop:SampleError:HighProb}
Fix $\delta \in (0,1)$ and $p \in (0,1)$. 
Let Assumptions \ref{Assumption:Moments}, \ref{Assumption:Source} and \ref{Assumption:Capacity} hold with $r \geq 1/2$ and $\eta_t = \eta t^{-\theta} $ with $\eta \kappa^2 \leq 1$, $\theta \in [0,1)$. 
The following holds with probability at least $1-\delta$ for any $t \in \mathbb{N}$
\begin{align*}
& \|\mathcal{S}_{\rho}(\xi_{t+1} - \mu_{t+1})\|_{\rho}\\
& \leq 
\widetilde{d}_1
\min \Big( \frac{1}{e(1-p)_{+}}, \log nm \Big)   
\frac{ \log(t)  }{ (nm)^{(1-p\gamma)/2}} 
(1 \vee (nm)^{-p}\eta t^{1-\theta}  \vee \eta t^{-\theta} ) \log^2 \frac{ \widetilde{d}_2}{\delta},
\end{align*}
with $\widetilde{d}_1 = 
768 \frac{ \kappa^2 \|\mathcal{T}_{\rho}\|^{1/2}(\|\mathcal{T}_{\rho}\| + 1)^{1/2} 
(R \kappa^{2r} + \sqrt{M})(\kappa + \sqrt{2 \sqrt{\nu}c_{\gamma}})}{1-\theta}$ 
and $\widetilde{d}_2 = 8\big(1 \vee \kappa^2 \frac{(c_{\gamma} + 1)}{\|\mathcal{T}_{\rho}\|}\big)$.
\end{proposition}
\begin{proof}
Fix $\delta \in (0,1)$ and set $\lambda = (nm)^{-p}$ with $p \in (0,1)$. 
Lemma \ref{Lem:SampleVariance:NBound} implies that with probability at least $1-\frac{ \delta}{2}$ the following holds for any $k \in \mathbb{N}$ 
\begin{align*}
\|(\mathcal{T}_{\rho} + \lambda I )^{-1/2} N_{k} \|_{H} 
\leq 
4(R \kappa^{2r} + \sqrt{M})\bigg(\kappa + \sqrt{2 \sqrt{\nu}c_{\gamma}}\bigg)
\frac{\log \frac{8}{\delta}}{(nm)^{(1-p\gamma)/2}}.
\end{align*}
Similarly,  Lemma \ref{Lem:Regularised:Operator:Bound} implies that the following holds with probability at least $1-\frac{\delta}{2}$ 
\begin{align*}
& \| \mathcal{T}_{\rho}^{1/2}(\mathcal{T}_{\mathbf{x}} + \lambda I )^{-1/2}\|^2
\leq
\| \mathcal{T}_{\rho,\lambda}^{1/2}(\mathcal{T}_{\mathbf{x}} + \lambda I )^{-1/2}\|^2 \\
& \leq 
48 \kappa^2  \min \bigg( \frac{1}{e(1-p)_{+}}, \log nm \bigg) 
\log \frac{8 \kappa^2 (c_{\gamma} + 1)}{\delta \|\mathcal{T}_{\rho}\|}.
\end{align*}
Following \cite{lin2017optimal}, the series can be bounded as follows
\begin{align*}
&  \sum_{k=1}^{t-1} \frac{\eta_{k}   }{2\sum_{i=k+1}^{t} \eta_i } 
+ \lambda \sum_{k=1}^{t-1} \eta_{k}  
+ \|\mathcal{T}_{\rho}\|^{1/2}(\|\mathcal{T}_{\rho}\| + \lambda)^{1/2} \eta_t\\
 & \leq 
2 \log(t) + \frac{ \lambda \eta t^{1-\theta}}{1-\theta} + \|\mathcal{T}_{\rho}\|^{1/2}(\|\mathcal{T}_{\rho}\| + 1)^{1/2} \eta t^{-\theta}\\
& \leq 
\frac{ 4 \|\mathcal{T}_{\rho}\|^{1/2}(\|\mathcal{T}_{\rho}\| + 1)^{1/2}  \log(t) }{1-\theta}
(1 \vee (\lambda \eta t^{1-\theta)}) \vee (\eta t^{-\theta}) ),
\end{align*}
where we used $\lambda = (nm)^{-p} \leq 1$ to get $(\|\mathcal{T}_{\rho}\| + \lambda)^{1/2}\leq(\|\mathcal{T}_{\rho}\| + 1)^{1/2}$. 
Plugging everything into Lemma \ref{Lem:SampleVariance:Bound} and using a union bound we obtain that the result holds with probability at least \\ $1-\frac{\delta}{2} - \frac{\delta}{2}  = 1-\delta$. 
\end{proof}
Proposition \ref{prop:SampleError:HighProb} gives a bound that holds with high probability. We make use of the following lemma to derive a bound in expectation. 
\begin{lemma}{\citep[Appendix Lemma C.1]{blanchard2018optimal}}
\label{Lem:TailBoundToExpectation}
Let $F:(0,1]\rightarrow \mathbb{R}_{+}$ be a monotone, non-increasing, continuous function and $V$ a non-negative real-valued random variable such that 
$$
	\P[ V > F(t)] \leq t, \quad \forall t \in (0,1].
$$
Then we have
$
	\E[V] \leq \int_{0}^{1} F(t) dt.
$
\end{lemma}

The following theorem presents the final bound for the expected value of the Sample Variance term.  
\begin{theorem}
\label{thm:SampleVarianceBound}
Let Assumptions \ref{Assumption:Moments}, \ref{Assumption:Source}, \ref{Assumption:Capacity} hold with  $r \geq 1/2$, $p \in (0,1)$ and $\eta_{t} = \eta t^{-\theta}$ for all $t \in \mathbb{N}$  with $\eta \in (0,\kappa^{-2}]$,  $\theta \in [0,1)$. Then for following holds for all $t \in \mathbb{N}$:
\begin{align*}
& \E[ \|\mathcal{S}_{\rho}(\xi_{t} - \mu_{t})\|^2_{\rho}]\\
& \leq 
\widetilde{d}_3 
\min \bigg( \frac{1}{e(1-p)_{+}}, \log nm \bigg)^2   
\frac{ \log^2(t) }{ (nm)^{(1-p\gamma)}}
\bigg( 
1 \vee ((nm)^{-p} \eta t^{1-\theta})^{2} \vee t^{-2} (\eta t^{1-\theta})^{2} \bigg),
\end{align*}
with $\widetilde{d}_{3} = 64 \widetilde{d}_{1}^2 \log^4 \widetilde{d}_{2}$ and with $\widetilde{d}_1$, $\widetilde{d}_{2}$ defined as in Proposition \ref{prop:SampleError:HighProb}.
\end{theorem}
\begin{proof}
Consider the term $\|\mathcal{S}_{\rho}(\xi_{t} - \mu_{t})\|^2_{\rho}$. Utilising the high-probability bound in Proposition \ref{prop:SampleError:HighProb} as well as Lemma \ref{Lem:TailBoundToExpectation}, the expectation of the squared norm can be bounded as
\begin{align*}
& \E[ \|\mathcal{S}_{\rho}(\xi_{t} - \mu_{t})\|^2_{\rho}]\\
& \leq 
\widetilde{d}_1^2 
\min \bigg( \frac{1}{e(1-p)_{+}}, \log nm \bigg)^2   
\frac{ \log^2(t) }{ (nm)^{(1-p\gamma)}}
\bigg(
 1 \vee ((nm)^{-p} \eta t^{1-\theta})^{2} \vee t^{-2} (\eta t^{1-\theta})^{2}
 \bigg) \\
& \quad\  \times  
\int_{0}^{1} \log^4 \frac{\widetilde{d}_{2}}{\delta} d \delta.
\end{align*}
The result follows by using the bound
$\int_{0}^{1}
\log^4 \frac{\widetilde{d}_{2}}{\delta} d \delta \leq 64 \log^4(\widetilde{d}_2) $.
\end{proof}

\subsection{Network  Error}
\label{Sec:Proof:Distributed}
In this section we develop the bound for the Network Error term.
The following lemma shows that the error can be decomposed into terms similar to $\{N_{k}\}_{k \in \mathbb{N}}$ defined in \eqref{equ:NTerm} for the Sample Variance. 
\begin{lemma}
\label{Lem:DistributedError:Decomposed}
For all $t \in \mathbb{N}$ we have 
\begin{align*}
& \|\mathcal{S}_{\rho}(\omega_{t+1,v} - \xi_{t+1,v}) \|_{\rho} = 
\bigg\| \sum_{k=1}^{t} \eta_{k} \sum_{w_{t:k} \in V^{t-k+1}} 
	\!\!\big( 
	P_{v w_{t:k}} \!-\! \frac{1}{n^{t-k+1}}  \big)
	\mathcal{T}^{1/2}_{\rho}
	\Pi_{t:k+1}(\mathcal{T}_{\mathbf{x}_{w_{t:k+1}}}) N_{k,w_{k}}
	\bigg \|_{H},
\end{align*}
where 
\begin{align}
\label{Equ:LocalErrorTerm}
N_{k,v} := 
(\mathcal{T}_{\rho}\mu_k - \mathcal{S}^{\star}_{\rho} f_{\rho}) - (\mathcal{T}_{\mathbf{x}_{v} } \mu_{k} - \mathcal{S}^{\star}_{\mathbf{x}_{v}} \mathbf{y}_{v}), 
\quad \forall k \in \mathbb{N}, \,\, v \in V.
\end{align}
\end{lemma}
\begin{proof}
For $t \geq 1$  the difference between the iterates $\omega_{t+1,v} - \mu_{t+1}$ can be written as follows 
\begin{align*}
\omega_{t+1,v}  - \mu_{t+1}
& = 
\sum_{w \in V} 
P_{vw} 
\Big( 
\omega_{t,w} - \mu_{t}
+ 
\eta_{t}
\big\{
(\mathcal{T}_{\rho} \mu_{t} - \mathcal{S}_{\rho}^{\star} f_{\rho}) 
- 
(\mathcal{T}_{\mathbf{x}_{w}} \omega_{t,w} - \mathcal{S}^{\star}_{\mathbf{x}_{w}} \mathbf{y}_{w} )
\big\}
\Big)\\
& = 
\sum_{w \in V} 
P_{vw} 
\Big( 
(I - \eta_{t} \mathcal{T}_{\mathbf{x}_{w}})(\omega_{t,w} - \mu_{t})
+ 
\eta_t
\underbrace{ 
\big\{
(\mathcal{T}_{\rho} \mu_{t} - \mathcal{S}_{\rho}^{\star} f_{\rho}) 
- 
(\mathcal{T}_{\mathbf{x}_{w}} \mu_{t} - \mathcal{S}^{\star}_{\mathbf{x}_{w}} \mathbf{y}_{w} )
\big\}
}_{N_{t,w}}
\Big)\\
&= 
\sum_{w \in V} 
P_{vw} 
\Big( 
(I - \eta_{t} \mathcal{T}_{\mathbf{x}_{w}})(\omega_{t,w} - \mu_{t})
+ 
\eta_t
N_{t,w}
\Big).
\end{align*}
Unravelling the iterates and using $\omega_{1} = \mu_1 =0$ yield
\begin{align*}
\omega_{t+1,v} \!-\! \mu_{t+1} 
& = 
\sum_{w_{t:1} \in V^{t}} 
P_{v w_{t:1}} \Pi_{t:1}(\mathcal{T}_{\mathbf{x}_{w_{t:1}}})(\omega_{1} \!-\! \mu_1) + 
 \sum_{k=1}^{t} \eta_{k}\!
\sum_{w_{t:k} \in V^{t}}\!\!
P_{v w_{t:k}} \Pi_{t:k+1}(\mathcal{T}_{\mathbf{x}_{w_{t:k+1}}}) N_{k,w_{k}}\\
& = 
 \sum_{k=1}^{t} \eta_{k}
\sum_{w_{t:k} \in V^{t-k+1}}
P_{v w_{t:k}} \Pi_{t:k+1}(\mathcal{T}_{\mathbf{x}_{w_{t:k+1}}}) N_{k,w_{k}}.
\end{align*}
The iterates $\xi_{t+1,v} - \mu_{t+1}$ are similarly written and unravelled using $\xi_{1,v} = 0$:
\begin{align*}
\xi_{t+1,v} - \mu_{t+1} 
& = 
\sum_{w \in V} \frac{1}{n} 
\Big( 
(I - \eta_{t} \mathcal{T}_{\mathbf{x}_{w}})(\xi_{t,w}- \mu_{t})  + 
\eta_{t} N_{t,w}
\Big)\\
& = 
\sum_{k=1}^{t} \eta_{k}  \sum_{w_{t:k} \in V^{t-k+1}} \frac{1}{n^{t-k+1}} 
\Pi_{t:k+1}(\mathcal{T}_{\mathbf{x}_{w_{t:k+1}}}) N_{k,w_{k}}.
\end{align*}
The deviation $\omega_{t+1} - \xi_{t+1,v}$ can then be written as follows 
\begin{align*}
\omega_{t+1,v} - \xi_{t+1,v}
& = 
 \sum_{k=1}^{t} \eta_{k} \sum_{w_{t:k} \in V^{t-k+1}} 
\bigg( 
P_{v w_{t:k}} - \frac{1}{n^{t-k+1}}  \bigg)
\Pi_{t:k+1}(\mathcal{T}_{\mathbf{x}_{w_{t:k+1}}}) N_{k,w_{k}}.
\end{align*}
Applying $\mathcal{S}_{\rho}$, taking norm $\|\cdot\|_{\rho}$ on both sides and using the isometry property yields the result.
\end{proof}
For $v, w \in V$ and $k \geq 1$, we want to exploit that the random variables $N_{k,v}$ and $N_{k,w}$ have zero mean, $\E[N_{k,v}] =0$, and are independent for $v \not= w$. To do so we add and subtract $\Pi_{t:k+1}(\mathcal{T}_{\rho})$ inside the norm so  the following upper bound can be formed:
\begin{align}
\label{equ:DistributedError:Breakdown}
& \| \mathcal{S}_{\rho} ( \omega_{t+1,v}  - \xi_{t+1,v}) \|_{\rho}^2\\
\nonumber 
& \leq 
2 
\underbrace{ 
\bigg\| \sum_{k=1}^{t} \eta_{k} \sum_{w_{t:k} \in V^{t-k+1}} 
	\big( 
	P_{v w_{t:k}} - \frac{1}{n^{t-k+1}}  \big)
	\mathcal{T}^{1/2}_{\rho}\Pi_{t:k+1}(\mathcal{T}_{\rho}) N_{k,w_{k}}
	\bigg\|_{H}^2 
	}_{
		(\textbf{Population Covariance Error})^2
	}\\
	\nonumber
& + 
2 
\underbrace{
\bigg\| \sum_{k=1}^{t} \eta_{k} \sum_{w_{t:k} \in V^{t-k+1}} 
	\big( 
	P_{v w_{t:k}} - \frac{1}{n^{t-k+1}}  \big)
	\mathcal{T}^{1/2}_{\rho}
	\big( 
	\Pi_{t:k+1}(\mathcal{T}_{\mathbf{x}_{w_{t:k+1}}})
	- \Pi_{t:k+1}(\mathcal{T}_{\rho})
	) N_{k,w_{k}}
	\bigg\|_{H}^2
	}_{ (\textbf{Residual Empirical Covariance Error})^2
	}.
\end{align}
The \textbf{Population Covariance Error} (\textbf{Pop. Cov. Error}) will be controlled by using the independence of the terms $\{N_{k,w}\}_{w \in V}$. The \textbf{Residual Empirical Covariance Error} (\textbf{Resid. Emp. Cov. Error}) will be analysed by decomposing it into terms that concentrate to zero sufficiently quickly. 

The following lemma, similar to Lemma \ref{Lem:SampleVariance:NBound} for the sample variance, gives concentration rates for the quantities held by the individual agents.  
\begin{lemma}
\label{lem:DistributedError:Bounds}
Fix $v \in V$. Let Assumptions \ref{Assumption:Moments}, \ref{Assumption:Source}, \ref{Assumption:Capacity} hold with $r \geq 1/2$ and $\{N_{s,v}\}_{s \in \mathbb{N}}$ be defined as in \eqref{Equ:LocalErrorTerm}. For any $\lambda > 0$, with probability at least $1-\delta$, the following holds for all $k \in \mathbb{N}$:
\begin{align}
\|(\mathcal{T}_{\rho} + \lambda I )^{-1/2} N_{k,v} \|_{H} 
\leq 
4(R \kappa^{2r} + \sqrt{M})
\bigg(
\frac{\kappa}{m \sqrt{\lambda}} 
+
\frac{\sqrt{2 \sqrt{\nu} c_{\gamma}}}{\sqrt{ m \lambda^{\gamma}}} 
\bigg)
\log\frac{4}{\delta}.
\label{equ:lem:DistributedError:1}
\end{align}
Let $\|\cdot\|_{HS}$ denote the Hilbert-Schmidt norm of a bounded operator from $H$ to $H$. The following holds with probability at least $1-\delta$:
\begin{align}
\| (\mathcal{T}_{\rho} + \lambda I )^{-1/2}(\mathcal{T}_{\rho} - \mathcal{T}_{\mathbf{x}_{v}})
\|_{HS}
\leq 
2 \kappa 
\bigg( 
\frac{2 \kappa}{m \sqrt{\lambda}} + \frac{\sqrt{c_{\gamma}}}{\sqrt{m \lambda^{\gamma}}} 
\bigg) \log \frac{4}{\delta}.
\label{equ:lem:DistributedError:2}
\end{align}
\end{lemma}
\begin{proof}
Both inequalities arise from concentration results for random variables in Hilbert spaces used in \cite{caponnetto2007optimal} and based on results in \cite{pinelis1986remarks}. Inequalities (\ref{equ:lem:DistributedError:1},\ref{equ:lem:DistributedError:2}) come directly from \cite[Lemma 18]{lin2017optimal}, where in particular \eqref{equ:lem:DistributedError:2} was used to prove  \eqref{equ:lem:DistributedError:1}.
\end{proof}
We now move on to establish bounds for the \textbf{Population Covariance Error} term and the \textbf{Residual Empirical Covariance Error} term within the following two sections, Section \ref{sec:DecentralisedError:Terma} and Section \ref{sec:DecentralisedError:Termb}, respectively. Section \ref{Sec:Network Error bound} then brings together the previously developed results to establish a bound for the Network Error term.

We will need the following lemma, taken from \cite[Lemma 15]{lin2017optimal}, which itself follows \cite{ying2008online,tarres2014online}. 
\begin{lemma}
\label{Lem:OperatorNorm}
Let $\mathcal{L}$ be a compact, positive operator on a separable Hilbert Space $H$. Assume that $\eta \|\mathcal{L}\| \leq 1$. For $t \in \mathbb{N}$, $a > 0$ and any non-negative integer $k \leq t-1$ we have 
\begin{align*}
\|\Pi_{t:k+1}(\mathcal{L}) \mathcal{L}^{a} \| \leq 
\bigg( 
\frac{a}{e \sum_{j=k+1}^{t} \eta_{j}} \bigg)^{a}.
\end{align*}
\end{lemma}
\begin{proof}
The proof in \cite[Lemma 15]{lin2017optimal} considers this result with  $a=r$. 
The proof for more general $a > 0$ follows the same steps.
\end{proof}

\subsubsection{Analysis of \textbf{Population Covariance Error}}
\label{sec:DecentralisedError:Terma}
In this section we develop a bound for the \textbf{Population Covariance Error} term in \eqref{equ:DistributedError:Breakdown}. The final result is presented in Lemma \ref{lem:DistributedError:Terma}.

The following proposition bounds the expectation of $(\textbf{Population Covariance Error})^2$ by a series involving the products of (deterministic) operators $\{\mathcal{T}_{\rho}^{1/2}\Pi_{t:k+1}(\mathcal{T}_{\rho})\}$, as a function of the step size, the largest eigenvalue in absolute value of the gossip matrix $P$, and the random variables $\{N_{k,w}\}$. 
\begin{proposition}
\label{Prop:DistributedError:Expectation}
For any $t \in \mathbb{N}$ and $v \in V$ we have
\begin{align*}
& \E \bigg[\bigg\| \sum_{k=1}^{t} \eta_{k} \sum_{w_{t:k} \in V^{t-k+1}} 
	\bigg( 
	P_{v w_{t:k}} - \frac{1}{n^{t-k+1}}  \bigg)
	\mathcal{T}^{1/2}_{\rho}
	\Pi_{t:k+1}(\mathcal{T}_{\rho}) N_{k,w_{k}}
	\bigg\|_{H}^2 \bigg]\\
&  \leq 
\E \bigg[ \bigg( 
\sum_{k=1}^{t} \sigma_2^{t-k+1} \eta_k 
\|\mathcal{T}^{1/2}_{\rho} \Pi_{t:k+1}(\mathcal{T}_{\rho}) N_{k,v}\|_{H}
\bigg)^2 \bigg].
\end{align*}
\end{proposition} 

\begin{proof}
Fix $t \in \mathbb{N}$ and $v \in V$. Let us introduce the notation 
$\Delta(w_{t:k}): = \big( P_{v w_{t:k}} - \frac{1}{n^{t-k+1}}  \big) $. 
Expanding the square and taking the expectation we get 
\begin{align*}
& \E \bigg[\bigg\| \sum_{k=1}^{t} \eta_{k} \sum_{w_{t:k} \in V^{t-k+1}} 
	\bigg( 
	P_{v w_{t:k}} - \frac{1}{n^{t-k+1}}  \bigg)
	\mathcal{T}^{1/2}_{\rho}
	\Pi_{t:k+1}(\mathcal{T}_{\rho}) N_{k,w_{k}}
	\bigg\|_{H}^2 \bigg]\\
& = 
\sum_{k,k^{\prime}=1}^{t}\!
\eta_{k} \eta_{k^{\prime}}\!\!
\sum_{\substack{ w_{t:k} \in V^{t-k+1}  \\ w^{\prime}_{t:k^{\prime}} \in V^{t-k^{\prime} +1}} } \!\!
\Delta(w_{t:k}) \Delta(w^{\prime}_{t:k^\prime})
\E
\langle 
\mathcal{T}^{1/2}_{\rho} \Pi_{t:k+1}(\mathcal{T}_{\rho}) N_{k,w_{k}}
,
\mathcal{T}^{1/2}_{\rho} \Pi_{t:k^{\prime}+1}(\mathcal{T}_{\rho}) N_{k^{\prime},w^{\prime}_{k^{\prime}}}
\rangle_{H}\\
& = 
\sum_{k,k^{\prime}=1}^{t}
\eta_{k} \eta_{k^{\prime}}
\E
\langle 
\mathcal{T}^{1/2}_{\rho} \Pi_{t:k+1}(\mathcal{T}_{\rho}) N_{k,v}
,
\mathcal{T}^{1/2}_{\rho} \Pi_{t:k^{\prime}+1}(\mathcal{T}_{\rho}) N_{k^{\prime},v}
\rangle_{H} 
\sum_{\substack{ w_{t:k} \in V^{t-k+1}  \\ w^{\prime}_{t:k^{\prime}} \in V^{t-k^{\prime} +1} \\
w_{k} = w^{\prime}_{k^{\prime}} } }
\Delta(w_{t:k}) \Delta(w^{\prime}_{t:k^\prime}).
\end{align*}
The last identity follows from the fact that the samples held by agents are independent and identically distributed. As the agents' datasets are independent, the inner products are zero for $k,k^{\prime} \in \{1,\dots,t\}$ whenever the final elements of the paths $w_{t:k}$ and $w_{t:k^{\prime}}^{\prime}$ do not coincide, i.e.
$$
	\E
	\langle 
\mathcal{T}^{1/2}_{\rho} \Pi_{t:k+1}(\mathcal{T}_{\rho}) N_{k,w_{k}}
,
\mathcal{T}^{1/2}_{\rho} \Pi_{t:k^{\prime}+1}(\mathcal{T}_{\rho}) N_{k^{\prime},w^{\prime}_{k^{\prime}}}
\rangle_{H}
= 0  \text{ if } w_{k} \not= w^{\prime}_{k^{\prime}}.
$$
As the agents' datasets are identically distributed, the expectation of the inner products can be taken outside the sum over the paths. The sum over all pairs of paths that intersect at the final node can be simplified as follows:
\begin{align*}
& \sum_{\substack{ w_{t:k} \in V^{t-k+1}  \\ w^{\prime}_{t:k^{\prime}} \in V^{t-k^{\prime} +1} \\
w_{k} = w^{\prime}_{k^\prime} } }
\Delta(w_{t:k}) \Delta(w^{\prime}_{t:k^\prime})\\
& = 
\sum_{ 
\substack{ 
w_{k}, w^{\prime}_{k^{\prime}} \in V  \\ 
w_k = w_{k^{\prime}}^{\prime} }}
\sum_{w_{t:k+1}  \in V^{t-k}}
\sum_{w^{\prime}_{t:k+1}  \in V^{t-k^{\prime}}}
\bigg(P_{v w_{t:k}} - \frac{1}{n^{t-k+1}}\bigg)
\bigg(P_{v w^{\prime}_{t:k^{\prime} }} - \frac{1}{n^{t-k^{\prime}+1}}\bigg) \\
& = 
\sum_{ w \in V  }
\bigg((P^{t-k+1})_{vw} - \frac{1}{n}\bigg)
\bigg((P^{t-k^{\prime}+1})_{vw} - \frac{1}{n}\bigg).
\end{align*}
For each $v \in V$ let $e_{v} \in \mathbb{R}^{n}$ denote the vector of all zeros but a $1$ in the place aligned with agent $v$. The summation can be further simplified by utilising the assumption that $P$ is symmetric and doubly-stochastic, i.e.\ $P^{\top} = P$ and $P \mathbf{1} = \mathbf{1}$.
By the eigendecomposition of the gossip matrix $P$, recall Section \ref{sec:DecentralisedLearning:DistributedGD}, for any $s > 0 $ we have 
$(P^{s})_{vv} = \sum_{l=1}^{n} \lambda_{l}^{s} u_{l,v}^2 = \frac{1}{n} + \sum_{l =2}^{n} \lambda_{l}^{s} u_{l,v}^2$. This yields the bound $|(P^{s})_{vv} - \frac{1}{n}| =
|\sum_{l =2}^{n} \lambda_{l}^{s} u_{l,v}^2| \leq \sigma_2^{s} \sum_{l=2}^{n} u_{l,v}^2 \leq \sigma_2^{s}$ where $\sigma_2 : = \max\{|\lambda_2|,|\lambda_{n}|\}$ is the second largest eigenvalue in absolute value. Bringing everything together, the expected norm of $(\textbf{Pop. Cov. Error})^2$ can be written and bounded as follows:
\begin{align*}
& \E \bigg[\bigg\| \sum_{k=1}^{t} \eta_{k} \sum_{w_{t:k} \in V^{t-k+1}} 
	\bigg( 
	P_{v w_{t:k}} - \frac{1}{n^{t-k+1}}  \bigg)
	\mathcal{T}^{1/2}_{\rho}
	\Pi_{t:k+1}(\mathcal{T}_{\rho}) N_{k,w_{k}}
	\bigg\|_{H}^2 \bigg]\\
& = 
\sum_{k,k^{\prime}=1}^{t}
\eta_{k} \eta_{k^{\prime}}
\E
\langle 
\mathcal{T}^{1/2}_{\rho} \Pi_{t:k+1}(\mathcal{T}_{\rho}) N_{k,v}
,
\mathcal{T}^{1/2}_{\rho} \Pi_{t:k^{\prime}+1}(\mathcal{T}_{\rho}) N_{k^{\prime},v}
\rangle_{H} 
\bigg(P^{2t-k-k^{\prime}+2}_{vv} - \frac{1}{n}\bigg)\\
& \leq 
\sum_{k,k^{\prime}=1}^{t}
\eta_{k} \eta_{k^{\prime}}
\E
|\langle 
\mathcal{T}^{1/2}_{\rho} \Pi_{t:k+1}(T_{\rho}) N_{k,v}
,
\mathcal{T}^{1/2}_{\rho} \Pi_{t:k^{\prime}+1}(\mathcal{T}_{\rho}) N_{k^{\prime},v}
\rangle_{H} |\,
\bigg|\bigg(P^{2t-k-k^{\prime}+2}_{vv} - \frac{1}{n}\bigg)\bigg|\\
& \leq 
\sum_{k,k^{\prime}=1}^{t}
\eta_{k} \eta_{k^{\prime}}
\E
\big[
\|\mathcal{T}^{1/2}_{\rho} \Pi_{t:k+1}(\mathcal{T}_{\rho}) N_{k,v}\|_{H}
\|\mathcal{T}^{1/2}_{\rho} \Pi_{t:k^{\prime}+1}(\mathcal{T}_{\rho}) N_{k^{\prime},v}\|_{H}
\big]
\sigma_2^{2t-k-k^{\prime}+2} \\
& = 
\E
\bigg[
\bigg( 
\sum_{k=1}^{t} \eta_{k} \sigma_2^{t-k+1}
\|\mathcal{T}^{1/2}_{\rho} \Pi_{t:k+1}(\mathcal{T}_{\rho}) N_{k,v}\|_{H}
\bigg)^2
\bigg],
\end{align*}
where we used Jensen's inequality
and the Cauchy-Schwarz inequality.
\end{proof}
The following lemma presents the final bound for the  \textbf{Population Covariance Error}. This result is established by utilising the series bound in  Proposition \ref{Prop:DistributedError:Expectation} to split the error into well-mixed and poorly-mixed terms, i.e.\ for $k$ such that $t-k \gtrsim 1/(1-\sigma_2)$ and $t-k \lesssim 1/(1-\sigma_2)$. The well-mixed terms are controlled using that  $\sigma_2^{t-k+1}$ is small. The poorly-mixed terms (there are $\sim1/(1-\sigma_2)$ of them) are controlled using both the concentration of the error terms $\{N_{k,w}\}_{k \geq 1, w \in V}$ as well as the contractive nature of the gradient updates, i.e.\ the operator norm of $\{\mathcal{T}_{\rho}^{1/2}\Pi_{t:k+1}(\mathcal{T}_{\rho})\}$ in Lemma \ref{Lem:OperatorNorm}. 
The contractive terms arising from the gradient updates are decreasing in the step size: larger steps achieve a faster contraction. However, each term within the Network Error series is scaled by the step size $\{\eta_{k}\}_{k \geq 1}$, i.e. the Network Error takes the form $\sum_{k=1}^{t} \sigma_2^{t-k+1} \eta_k [\cdots]$ where $[\cdots]$ indicates the right most terms.
To exploit this trade-off we introduce two free parameters $\alpha \in [0,1/2]$ and $\gamma^{\prime} \in [1,\gamma]$, which describe the degree to which the contraction is utilised. Specifically, $\alpha = 0$ and $\gamma^{\prime} = \gamma$ is the large step regime and, $\alpha=1/2$ and $\gamma^{\prime} = 1$ is the small step regime.

\begin{lemma}
\label{lem:DistributedError:Terma}
Let Assumptions \ref{Assumption:Moments}, \ref{Assumption:Source}, \ref{Assumption:Capacity} hold with $r \geq 1/2$, $\eta_{t} = \eta t^{-\theta}$ for $t \in \mathbb{N}$  with $\eta \kappa^2 \leq 1$ and $\theta \in [0,1)$. The following holds for any $v \in V$, 
$t/2 \geq \lceil \frac{  (1+r) \log(t)}{1-\sigma_2 } \rceil  =:  t^{\star}$, $\alpha \in [0,1/2]$ and $\gamma^{\prime} \in [1,\gamma]$:
\begin{align*}
   &  \E \bigg[ \bigg\| \sum_{k=1}^{t} \eta_{k} \sum_{w_{t:k} \in V^{t-k+1}} 
	\bigg( 
	P_{v w_{t:k}} - \frac{1}{n^{t-k+1}}  \bigg)
	\mathcal{T}^{1/2}_{\rho}
	\Pi_{t:k+1}(\mathcal{T}_{\rho}) N_{k,w_{k}}
	\bigg\|_{H}^2 \bigg] \\
&  \leq
\frac{ \widetilde{a} \log^2(4n) \log^2(t^{\star})}{m} 
\Big(  \eta^2 t^{-2r} 
\vee (m^{-1} (\eta t^{\star})^{1 + 2\alpha})  \vee (\eta t^{\star})^{\gamma^{\prime} + 2 \alpha}
\Big),
\end{align*}
where \\
$\widetilde{a} = \frac{ 1152 
(R \kappa^{2r} + \sqrt{M})^2 (\kappa + \sqrt{2 \sqrt{\nu} c_{\gamma^{\prime}}})^2 
(\|\mathcal{T}_{\rho}\| \vee 1)^2}{ \|\mathcal{T}_{\rho}\| \wedge \|\mathcal{T}_{\rho}\|^{\gamma^{\prime}} }
\Big[
    6 \Big( \frac{  \|\mathcal{T}_{\rho}^{\alpha}\| t^{-\alpha \theta}}{\alpha} 
    \! \vee \! 
    \frac{  t^{-(\alpha + 1/2)\theta } \|\mathcal{T}^{\alpha}_{\rho}\|
    }{ 1/2 + \alpha}
    \! \vee \! t^{-\theta} \|\mathcal{T}_\rho\| \Big)\mathbbm{1}_{ \{ \alpha \not= 0 \} }
    \! + \!
    10 \Big]^2$.
\end{lemma}
\begin{proof}
Consider the bound of \textbf{Population Covariance Error} in Proposition \ref{Prop:DistributedError:Expectation}. Let $\|\mathcal{T}_{\rho}\| \geq \lambda \geq 0 $,\\
$\widetilde{\lambda} \geq 0$ and for $c > 0 $ introduce the cutoff  $t^{\star} = \lceil \frac{c \log(t)}{1-\sigma_2} \rceil$.
For  $k =1,\dots,t$ and $v \in V$ we have
\begin{align*}
    \|\mathcal{T}^{1/2}_{\rho} \Pi_{t:k+1}(\mathcal{T}_{\rho}) N_{k,v}\|_{H}
    & \leq 
    \|\mathcal{T}^{1/2}_{\rho} \Pi_{t:k+1}(\mathcal{T}_{\rho})\mathcal{T}_{\rho,\lambda}^{1/2} \| \|\mathcal{T}_{\rho,\lambda}^{-1/2} N_{k,v}\|_{H}\\
    & \leq \|\mathcal{T}^{1/2}_{\rho} \Pi_{t:k+1}(\mathcal{T}_{\rho})
    \mathcal{T}_{\rho,\lambda}^{1/2} \|
    \max_{k=1,\dots,t} 
    \Big\{
     \|\mathcal{T}_{\rho,\lambda}^{-1/2} N_{k,v}\|_{H}
    \Big\},
\end{align*}
and similarly for $\widetilde{\lambda}$.
Let us split the summation at $k \leq t - t^{\star}-1$ and $k \geq t-t^{\star}$ using the bound above to obtain
\begin{align*}
&     \bigg( 
\sum_{k=1}^{t} \sigma_2^{t-k+1} \eta_k 
\|\mathcal{T}^{1/2}_{\rho} \Pi_{t:k+1}(\mathcal{T}_{\rho}) N_{k,v} \|_{H}
\bigg)^2  \\
&  \leq 2 \bigg( 
\underbrace{ 
\sum_{k=1}^{t-t^{\star}-1}
\sigma_2^{t-k+1} \eta_k 
\|\mathcal{T}^{1/2}_{\rho} \Pi_{t:k+1}(\mathcal{T}_{\rho}) \mathcal{T}_{\rho,\lambda}^{1/2} \|
}_{\textbf{Well-Mixed Network Error}}
\bigg)^2
\max_{k=1,\dots,t} 
    \Big\{
     \|\mathcal{T}_{\rho,\lambda}^{-1/2} N_{k,v}\|_{H}^2
    \Big\}\\ 
& \quad + 
 2
 \bigg( 
 \underbrace{ 
\sum_{k=t-t^{\star}}^{t}
\sigma_2^{t-k+1} \eta_k 
\|\mathcal{T}^{1/2}_{\rho} \Pi_{t:k+1}(\mathcal{T}_{\rho}) \mathcal{T}_{\rho,\widetilde{\lambda}}^{1/2}\| 
}_{ \textbf{Poorly-Mixed Network Error}}
\bigg)^2
\max_{k=1,\dots,t} 
    \Big\{
     \|\mathcal{T}_{\rho,\widetilde{\lambda}}^{-1/2} N_{k,v}\|_{H}^2
    \Big\}.
\end{align*}

The \textbf{Well-Mixed Network Error}  is controlled through $\sigma_2^{t-k+1}$ being small for $k \leq t-t^{\star}$. From $\|\Pi_{t:k+1}(\mathcal{T}_{\rho})\| \leq 1$ and $\lambda \leq \|\mathcal{T}_{\rho}\|$ we have $\|\mathcal{T}^{1/2}_{\rho} \Pi_{t:k+1}(\mathcal{T}_{\rho}) \mathcal{T}_{\rho,\lambda}^{1/2}\|_{H}  \leq 2\|\mathcal{T}_{\rho}\|$, and from $1/\log(1/\sigma_2) \leq 1/(1-\sigma_2)$ we have  $t^{\star} \geq \frac{c \log(t)}{-\log(\sigma_2)}$. These two facts allow  the \textbf{Well-Mixed Network Error} to be bounded as follows:
\begin{align*}
    \textbf{Well-Mixed Network Error} 
    \leq 
    2\|\mathcal{T}_{\rho}\|
    \eta \sum_{k=1}^{t-t^{\star}}\sigma_2^{t-k+1}k^{-\theta}
    \leq 
    2\eta \|\mathcal{T}_{\rho}\| \sum_{k=1}^{t-t^{\star}} 
    \sigma_2^{ \frac{c \log(t)}{-\log(\sigma_2)}}
    \leq
    2\eta \|\mathcal{T}_{\rho}\|  t^{1-c}.
\end{align*}

For the \textbf{Poorly-Mixed Network Error} let us consider the two cases $\alpha \in (0,1/2]$ and $\alpha = 0$ separately. Consider  $\alpha \in (0,1/2]$ first. 
Using Lemma \ref{Lem:OperatorNorm}\footnote{
The operator norm can be bounded  
$\|\mathcal{T}^{1/2}_{\rho} \Pi_{t:k+1} (\mathcal{T}_{\rho}) \mathcal{T}^{1/2}_{\rho,\lambda} \| \leq \sup_{x \in (0,\kappa^2) } \big\{
x^{1/2} (x+ \lambda )^{1/2} 
\prod_{\ell=k+1}^{t}(1 - \eta_\ell x)
\big\} \leq
\sup_{x \in (0,\kappa^2) }\big\{
x \prod_{\ell=k+1}^{t}(1 - \eta_\ell x)
\big\} + \sqrt{\lambda} \sup_{x \in (0,\kappa^2) } \big\{
x^{1/2}\prod_{\ell=k+1}^{t}(1 - \eta_\ell x)
\big\}.
$ Using techniques used to prove \citep[Lemma 15]{lin2017optimal}, these terms can be bounded as shown.
} we have, for $ t-1 \geq  k \geq 1$,
\begin{align*}
\|\mathcal{T}^{1/2}_{\rho} \Pi_{t:k+1}(\mathcal{T}_{\rho}) \mathcal{T}_{\rho,\widetilde{\lambda}}^{1/2}\| 
& \leq 
\|\mathcal{T}_{\rho }\Pi_{t:k+1}(\mathcal{T}_{\rho}) \| +
\sqrt{\widetilde{\lambda}} 
\|\mathcal{T}_{\rho }^{1/2}\Pi_{t:k+1}(\mathcal{T}_{\rho}) \|\\
& \leq 
\|\mathcal{T}_{\rho}^{\alpha}\| 
\|\mathcal{T}_{\rho }^{1-\alpha}\Pi_{t:k+1}(\mathcal{T}_{\rho}) \| 
+
\sqrt{\widetilde{\lambda}}
\|\mathcal{T}_{\rho}^{\alpha}\| 
\|\mathcal{T}_{\rho }^{1/2 - \alpha}\Pi_{t:k+1}(\mathcal{T}_{\rho}) \|\\ 
& \leq 
\|\mathcal{T}_{\rho}^{\alpha}\|  
\bigg( \frac{ 1 - \alpha }{e \sum_{j=k+1}^{t} \eta_j } \bigg)^{1 -\alpha}
+ 
\sqrt{\widetilde{\lambda}}
\|\mathcal{T}_{\rho}^{\alpha}\| 
\bigg( \frac{1/2 - \alpha }{ e \sum_{j=k+1}^{t} \eta_{j}}\bigg)^{1/2- \alpha}.
\end{align*}
When plugging the above into the \textbf{Poorly-Mixed Network Error}, summations of the form \\
$\sum_{k=t-t^{\star}}^{t-1} \frac{ \eta_{k}}{ ( \sum_{j=k+1}^{t} \eta_{j} )^{\beta}}$ appear for $\beta = 1-\alpha$ and $\beta = 1/2-\alpha$.
To bound these consider the following for $\beta \in [0,1)$ and $t \geq 2 t^{\star}$:
\begin{align*}
    \sum_{k=t-t^{\star}}^{t-1} \frac{ \eta_{k}}{\big( \sum_{j=k+1}^{t} \eta_j\big)^{\beta} }
    & = 
    \eta^{1-\beta}
    \sum_{k=t-t^{\star}}^{t-1} 
    \frac{ k^{-\theta}}{\big( \sum_{j=k+1}^{t} j^{-\theta}\big)^{\beta} }\\
    & \leq 
    \eta^{1-\beta} t^{\theta \beta}
    \sum_{k=t-t^{\star}}^{t-1} 
    \frac{ k^{-\theta}}{\big( t- k \big)^{\beta} }\\
    & \leq
    \frac{ \eta^{1-\beta} t^{\theta \beta}}{ (t - t^{\star})^{\theta} }
    \sum_{k=t-t^{\star}}^{t-1} 
    \frac{ 1}{\big( t- k \big)^{\beta} }\\
    & = 
    \frac{ \eta^{1-\beta} t^{\theta \beta}}{ (t - t^{\star})^{\theta} }
    \sum_{k=1 }^{t^{\star}} 
    \frac{ 1}{ k^{\beta} }\\
    & \leq 2 \eta^{1-\beta} t^{\theta (\beta-1) }
    \frac{ (t^{\star})^{1-\beta}}{1-\beta},
\end{align*}
where the last inequality follows from an integral bound as well as using that
$\frac{ t^{\theta \beta}}{ (t - t^{\star})^{\theta} } 
= \frac{ t^{\theta (\beta-1)}}{ (1 - \frac{t^{\star}}{t} )^{\theta} }
\leq 
2t^{\theta (\beta-1)}$ from $t \geq 2 t^{\star}$.
Splitting the summation at $k=t$, plugging the above two bounds into the \textbf{Poorly-Mixed Network Error} term and using $(\eta t^{\star})^{\alpha} \geq \eta$ from $\eta \leq \kappa^{-2} \leq 1$
yields a bound for $\alpha \in (0,1/2]$:
\begin{align*}
    & \textbf{Poorly-Mixed Network Error}\\
    & \leq 
    \frac{ 2 \|\mathcal{T}_{\rho}^{\alpha}\| t^{-\alpha \theta}}{\alpha} 
    (\eta t^\star)^{\alpha}
    + 
    \frac{ 2 t^{-(\alpha + 1/2)\theta } 
    \|\mathcal{T}^{\alpha}_{\rho}\|}{1/2 + \alpha} 
    \sqrt{\widetilde{\lambda}} (\eta t^{\star})^{1/2 + \alpha} 
    + \sqrt{2} \eta t^{-\theta}  \|\mathcal{T}_{\rho}\|  \\
    & \leq 
    6 \bigg( \frac{  \|\mathcal{T}_{\rho}^{\alpha}\| t^{-\alpha \theta}}{\alpha}
    \vee 
    \frac{  t^{-(\alpha + 1/2)\theta } \|\mathcal{T}^{\alpha}_{\rho}\|
    }{ 1/2 + \alpha}
    \vee t^{-\theta} \|\mathcal{T}_{\rho}\|
    \bigg)
    ( (\eta t^\star)^{\alpha} \vee \sqrt{\widetilde{\lambda}} (\eta t^{\star})^{1/2 + \alpha}).
\end{align*}
Now consider the case  $\alpha =0$. The summation for $\beta = 1$ in this case is bounded following the previous steps  
\begin{align*}
    \sum_{k=t-t^{\star}}^{t-1} \frac{ \eta_{k}}{\big( \sum_{j=k+1}^{t} \eta_j\big) }
    \leq 
    \frac{t^{\theta} }{(t- t^{\star})^{\theta}}
    \sum_{k=t-t^{\star}}^{t-1} \frac{1}{(t-k)}
    \leq 
    2^{1+\theta}  \log(t^{\star}),
\end{align*}
leading to the \textbf{Poorly-Mixed Network Error} bounded as  for $\alpha = 0$ from $\eta \|\mathcal{T}_{\rho}\| \leq 1$:
\begin{align*}
    \textbf{Poorly-Mixed Network Error}
    & \leq 
    2^{1+\theta} \log(t^{\star})
     + 
     4 t^{-\theta/2 } 
    \sqrt{\widetilde{\lambda}} (\eta t^{\star})^{1/2} 
    +\sqrt{2} \eta t^{-\theta} \|\mathcal{T}_{\rho}\|\\
    & \leq
    10 \log(t^{\star}) 
    (1 \vee (\sqrt{\widetilde{\lambda}} (\eta t^{\star})^{1/2})).
\end{align*}
Combining the two bounds for $\alpha=0$ and $\alpha \in (0,1/2]$ gives 
\begin{align*}
    & \textbf{Poorly-Mixed Network Error}\\
    & \leq 
    \log(t^{\star}) \bigg[
    6 \bigg( \frac{  \|\mathcal{T}_{\rho}^{\alpha}\| t^{-\alpha \theta}}{\alpha} 
    \!\vee\!
    \frac{  t^{-(\alpha + 1/2)\theta } \|\mathcal{T}^{\alpha}_{\rho}\|
    }{ 1/2 + \alpha}
    \!\vee\! t^{-\theta} \|\mathcal{T}_{\rho}\|
    \bigg)\mathbbm{1}_{\{\alpha \not= 0\}}
    \! + \! 
    10
    \bigg] 
    ( (\eta t^\star)^{\alpha} \vee \sqrt{\widetilde{\lambda}} (\eta t^{\star})^{1/2 + \alpha}).
\end{align*}

We now consider the terms
$\max_{k=1,\dots,t} \{\|\mathcal{T}_{\rho,\lambda}^{-1/2} N_{k,v}\|_{H}^2\}$
for both $\lambda$ and $\widetilde{\lambda}$. We use the high-probability bounds of Lemma \ref{lem:DistributedError:Bounds} to uniformly control $\|\mathcal{T}^{-1/2}_{\rho,\lambda} N_{k,v}\|_{H}^2$ for all $k=1,\dots,t$ and $v \in V$. For $w \in V$, let $\delta_{w} = \frac{\delta}{n}$. With probability at least $1-\delta_{w}$ the following holds for all $k=1,\dots,t$ and $\gamma^{\prime} \in [1,\gamma]$:
\begin{align*}
\|\mathcal{T}^{-1/2}_{\rho,\lambda} N_{k,w}\|_{H}^2
\leq 
16(R \kappa^{2r} + \sqrt{M})^2
\bigg(
\frac{\kappa}{m \sqrt{\lambda}} 
+
\frac{\sqrt{2 \sqrt{\nu} c_{\gamma^{\prime}}}}{\sqrt{ m \lambda^{\gamma^{\prime}}}} 
\bigg)^2
\log^2 \frac{4n}{\delta}. 
\end{align*}
We note that if the capacity assumption holds for $\gamma$, then it also holds for all $\gamma^{\prime} \in [1,\gamma]$.  
Applying a union bound, we get that the above holds with probability at least $ 1- \sum_{v \in V} \delta_{v} = 1- \delta$ for all $w \in V$ and $k=1,\dots,t$. Using Lemma \ref{Lem:TailBoundToExpectation}, the expectation of the maximum can be bounded for any $v \in V$ and $\gamma^{\prime} \in [1,\gamma]$ as follows:
\begin{align*}
& \E \big[ \max_{k =1,\dots,t}
\big\{ 
\|\mathcal{T}^{-1/2}_{\rho,\lambda} N_{k,v}\|_{H}^2
\big\}
\big]\\
& \leq 
16 
 (R \kappa^{2r} + \sqrt{M})^2
\bigg(
\frac{\kappa}{m \sqrt{\lambda}} 
+
\frac{\sqrt{2 \sqrt{\nu} c_{\gamma^{\prime}}}}{\sqrt{ m \lambda^{\gamma^{\prime}}}} 
\bigg)^2 \int_{0}^{1} \log^2 \frac{4n }{\delta} d\delta\\
& \leq 
96
 (R \kappa^{2r} + \sqrt{M})^2
\bigg(
\frac{\kappa}{m \sqrt{\lambda}} 
+
\frac{\sqrt{2 \sqrt{\nu} c_{\gamma^{\prime}}}}{\sqrt{ m \lambda^{\gamma^{\prime}}}} 
\bigg)^2
\log^2 4n,
\end{align*}
where we used $\int_{0}^{1} \log^2 \frac{4n }{\delta} d\delta \leq 6 \log^2 4n$.

Bringing together the bounds for the \textbf{Poorly-} and \textbf{Well-Mixed Network Error} with the above bound for the quantity $\E \big[ \max_{k =1,\dots,t}
\big\{ 
\|\mathcal{T}^{-1/2}_{\rho,\lambda} N_{k,v}\|_{H}^2
\big\}
\big]$ yields 

\begin{align*}
    & \E \bigg[ \bigg( 
\sum_{k=1}^{t} \sigma_2^{t-k+1} \eta_k 
\|\mathcal{T}^{1/2}_{\rho} \Pi_{t:k+1}(\mathcal{T}_{\rho}) N_{k,v}\|_{H}
\bigg)^2 \bigg]\\
&  \leq
96\log^2(4n) \log^2(t^{\star})  
(R \kappa^{2r} + \sqrt{M})^2  \\
& \quad \times
\Bigg( 8 \|\mathcal{T}_{\rho}\|^2 
\Big(
\frac{\kappa}{m \sqrt{\lambda}} 
+
\frac{\sqrt{2 \sqrt{\nu} c_{\gamma}}}{\sqrt{ m \lambda^{\gamma}}} 
\Big)^2  \eta^2 t^{2(1-c)} \\
&  \quad\quad + 
2\Big[
    6 \Big( \frac{  \|\mathcal{T}_{\rho}^{\alpha}\| t^{-\alpha \theta}}{\alpha} 
    \vee 
    \frac{  t^{-(\alpha + 1/2)\theta } \|\mathcal{T}^{\alpha}_{\rho}\|
    }{ 1/2 + \alpha}
    \vee t^{-\theta} \|\mathcal{T}_{\rho}\| \Big)\mathbbm{1}_{\{\alpha \not= 0\}}
    \! + \!
    10
    \Big]^2 
    \Big(
\frac{\kappa}{m \sqrt{\widetilde{\lambda}}} 
+
\frac{\sqrt{2 \sqrt{\nu} c_{\gamma^{\prime}}}}{\sqrt{ m \widetilde{\lambda}^{\gamma^{\prime}}}} 
\Big)^2  \\
& \quad\quad \times \Big( (\eta t^{\star})^{2\alpha} \vee \widetilde{\lambda} (\eta t^{\star})^{1 + 2\alpha} \Big) 
\Bigg).
\end{align*}
Let $\lambda = \|\mathcal{T}_{\rho}\|$ and $\widetilde{\lambda} = \frac{\|\mathcal{T}_{\rho}\|}{\eta t^{\star}}$. The bound
\begin{align*}
\frac{1}{m \sqrt{\widetilde{\lambda}}} 
+
\frac{1}{\sqrt{ m \widetilde{\lambda}^{\gamma^{\prime}}}} 
& \leq 
 \frac{2}{\sqrt{m}}
 \bigg( \frac{1}{\sqrt{m \|\mathcal{T}_{\rho}\| (\eta t^\star)^{-1} }} 
 \vee 
 \frac{1}{ \|\mathcal{T}_{\rho}\|^{\gamma^{\prime}/2} (\eta t^\star)^{-\gamma^{\prime}/2} }
\bigg)\\
& \leq 
\frac{2}{\sqrt{m (\|\mathcal{T}_{\rho}\| \wedge \|\mathcal{T}_{\rho}\|^{\gamma^{\prime}}})}
 \Big( \sqrt{\eta t^{\star}/m} \vee (\eta t^{\star})^{\gamma^{\prime}/2} \Big)
\end{align*}
allows the expected squared series to be bounded as follows:
\begin{align*}
   & \E \bigg[ \bigg( 
\sum_{k=1}^{t} \sigma_2^{t-k+1} \eta_k 
\|\mathcal{T}^{1/2}_{\rho} \Pi_{t:k+1}(\mathcal{T}_{\rho}) N_{k,v}\|_{H}
\bigg)^2 \bigg]\\
&  \leq
\frac{\widetilde{a} \log^2(4n) \log^2 (t^{\star}) }{m} 
\Big( (\eta t^{1-c})^2 \vee 
 (m^{-1} (\eta t^{\star})^{1 + 2\alpha})  \vee (\eta t^{\star})^{\gamma^{\prime} + 2 \alpha}
\Big)  
\end{align*}
where \\
$\widetilde{a} = \frac{ 1152 
(R \kappa^{2r} + \sqrt{M})^2 (\kappa + \sqrt{2 \sqrt{\nu} c_{\gamma^{\prime}}})^2 
(\|\mathcal{T}_{\rho}\| \vee 1)^2}{ \|\mathcal{T}_{\rho}\| \wedge \|\mathcal{T}_{\rho}\|^{\gamma^{\prime}} }
\Big[
    6 \Big( \frac{  \|\mathcal{T}_{\rho}^{\alpha}\| t^{-\alpha \theta}}{\alpha} 
    \! \vee \! 
    \frac{  t^{-(\alpha + 1/2)\theta } \|\mathcal{T}^{\alpha}_{\rho}\|
    }{ 1/2 + \alpha}
    \! \vee \! t^{-\theta} \|\mathcal{T}_{\rho}\| \Big)\mathbbm{1}_{\{\alpha \not= 0\}}
    \!+\! 10
    \Big]^2$. 
    The choice $c = 1+ r$ yields the final result. 
\end{proof}

\subsubsection{Analysis of \textbf{Residual Empirical Covariance Error}}
\label{sec:DecentralisedError:Termb}

In this section we develop a bound for the \textbf{Residual Empirical Covariance Error} term in \eqref{equ:DistributedError:Breakdown}. The final result is presented in Lemma \ref{lem:HighProbBoundTermb}.

The following proposition writes the \textbf{Residual Empirical Covariance Error} in terms of a series of quantities that will be later controlled. 
\begin{proposition}
\label{prop:QuadraticExpansion}
Let $t \geq k+1$. For any $w_{t:k+1} \in V^{t-k}$ we have
\begin{align*}
& \Pi_{t:k+1}(\mathcal{T}_{\mathbf{x}_{w_{t:k+1}}})
 = 
\Pi_{t:k+1}(\mathcal{T}_{\rho}) + 
\sum_{j=k+1}^{t} \eta_{j} \Pi_{t:j+1}(
\mathcal{T}_{\rho})(\mathcal{T}_{\rho} - \mathcal{T}_{\mathbf{x}_{w_{j}}}) \Pi_{j-1:k+1}(\mathcal{T}_{\mathbf{x}_{w_{j-1:k+1}}}).
\end{align*}
\end{proposition}
\begin{proof}
Adding and subtracting $(I - \eta_{t} \mathcal{T}_{\rho}) 
\Pi_{t-1: k+1}(\mathcal{T}_{\mathbf{x}_{w_{t-1:k+1}}})$ and unravelling yields the following:
\begin{align*}
& \Pi_{t:k+1}(\mathcal{T}_{\mathbf{x}_{w_{t:k+1}}}) 
- 
\Pi_{t:k+1}(\mathcal{T}_{\rho}) \\
& = 
(I - \eta_t \mathcal{T}_{\mathbf{x}_{w_{t}}}) 
\Pi_{t-1:k+1}(\mathcal{T}_{\mathbf{x}_{w_{t-1:k+1}}})
- 
(I - \eta_{t} \mathcal{T}_{\rho}) \Pi_{t-1:k+1}(\mathcal{T}_{\rho})\\
& = 
(I - \eta_t \mathcal{T}_{\mathbf{x}_{w_{t}}}) 
\Pi_{t-1:k+1}(\mathcal{T}_{\mathbf{x}_{w_{t-1:k+1}}})
-
(I - \eta_{t} \mathcal{T}_{\rho}) 
\Pi_{t-1:k+1}(\mathcal{T}_{\mathbf{x}_{w_{t-1:k+1}}})\\
& \quad\ + 
(I - \eta_{t} \mathcal{T}_{\rho}) 
\Pi_{t-1:k+1}(\mathcal{T}_{\mathbf{x}_{w_{t-1:k+1}}})
- 
(I - \eta_{t} \mathcal{T}_{\rho}) \Pi_{t-1:k+1}(\mathcal{T}_{\rho})\\
& = \eta_{t}(\mathcal{T}_{\rho} - \mathcal{T}_{\mathbf{x}_{w_{t}}})
\Pi_{t-1:k+1}(\mathcal{T}_{\mathbf{x}_{w_{t-1:k+1}}})
+
(I - \eta_{t} \mathcal{T}_{\rho}) 
\big[
 \Pi_{t-1:k+1}(\mathcal{T}_{\mathbf{x}_{w_{t-1:k+1}}})
-\Pi_{t-1:k+1}(\mathcal{T}_{\rho})
 \big] 
 \\
 & = 
 \sum_{j=k+1}^{t} \eta_{j} \Pi_{t:j+1}(\mathcal{T}_{\rho})(
 \mathcal{T}_{\rho} - \mathcal{T}_{\mathbf{x}_{w_{j}}}) \Pi_{j-1:k+1}(\mathcal{T}_{\mathbf{x}_{w_{j-1:k+1}}}).
\end{align*}
\end{proof}
Applying Proposition \ref{prop:QuadraticExpansion} to the  \textbf{Residual Empirical Covariance Error} term, using the triangle equality, yields
\begin{align}
& \bigg\| \sum_{k=1}^{t} \eta_{k} \sum_{w_{t:k} \in V^{t-k+1}} 
	\Delta(w_{t:k})
	\mathcal{T}^{1/2}_{\rho}
	\big( 
	\Pi_{t:k+1}(\mathcal{T}_{\mathbf{x}_{w_{t:k+1}}})
	- \Pi_{t:k+1}(\mathcal{T}_{\rho})
	\big)N_{k,w_{k}} \bigg\|_{H} \nonumber \\
& \leq 
\sum_{k=1}^{t-1} \eta_{k}
	\sum_{w_{t:k} \in V^{t-k+1}} 
	|\Delta(w_{t:k})|
	\sum_{j=k+1}^{t} \eta_{j}\nonumber\\
	& \quad\ \times 
	\| \mathcal{T}_{\rho}^{1/2} \Pi_{t:j+1}(\mathcal{T}_{\rho})
	(\mathcal{T}_{\rho} - \mathcal{T}_{\mathbf{x}_{w_{j}}})
	\Pi_{j-1:k+1}(\mathcal{T}_{\mathbf{x}_{w_{j-1:k+1}}})
	N_{k,w_{k}}\|_{H},
\label{prop:equ:FirstOrderError}
\end{align}
where the quantity is zero in the case $k=t$.
For $j \in \{2,\dots,t-1\}$ the above includes  the quantity $\Pi_{t:j+1}(\mathcal{T}_{\rho})$. This can be  interpreted in a similar manner to the filter function  associated for gradient descent, see for instance \cite[Example 2]{lin2018optimal}. In this context it is used to control the growth of the above error term, which is absent in the case $j=t$. This yields the following proposition. 
\begin{proposition}
\label{prop:DistributedError:Bound1_new}
Let Assumptions \ref{Assumption:Moments}, \ref{Assumption:Source}, \ref{Assumption:Capacity} hold  with $r \geq 1/2$ and $\eta_t = \eta t^{-\theta}$ for $t \in \mathbb{N}$ with $\eta \kappa^2 \leq 1$, $\theta \in (0,1)$. Fix $\lambda,\widetilde{\lambda} > 0$ and $\delta \in (0,1)$. With probability at least $1-\delta$ the following hold: for any $t-1 \geq j  \geq k+1$ and path $w_{t:k} \in V^{t-k+1}$ we have
\begin{align}
 &\| \mathcal{T}_{\rho}^{1/2} \Pi_{t:j+1}(\mathcal{T}_{\rho})
	(\mathcal{T}_{\rho} - \mathcal{T}_{\mathbf{x}_{w_{j}}})
	\Pi_{j-1:k+1}(\mathcal{T}_{\mathbf{x}_{w_{j-1:k+1}}})
	N_{k,w_{k}}\|_{H}
	\nonumber \\
	\nonumber
	 & \leq 
  2 \kappa \|\mathcal{T}^{1/2}_{\rho,\widetilde{\lambda}}\| 
\bigg( 
\frac{1}{ \sum_{i=j+1}^{t} \eta_i } + \bigg( \frac{\lambda}{  \sum_{i=j+1}^{t} \eta_i } \bigg)^{1/2} \bigg)   
\bigg( 
\frac{2 \kappa}{m \sqrt{ \lambda} } 
+ \frac{\sqrt{c_{\gamma}}}{\sqrt{m \lambda^{\gamma}}} 
\bigg)
\log \bigg(  \frac{4 n}{\delta} \bigg) \\
& \quad\ \times \max_{w \in V} \big\{ 
\|\mathcal{T}^{-1/2}_{\rho,\widetilde{\lambda}} N_{k,w} \|_{H}
\big\}\label{equ:DistributedError:Bound1_new:1},
\end{align}
for any $t-1 \geq k \geq 1$ and nodes $w_{t},w_{k} \in V$ 
\begin{align}
& \| \mathcal{T}_{\rho}^{1/2} (\mathcal{T}_{\rho} - \mathcal{T}_{\mathbf{x}_{w_{t}}})
N_{k,w_{k}}\|_{H}
\nonumber\\
\label{equ:DistributedError:Bound1_new:2}
& \leq 
2 \kappa 
\|\mathcal{T}_{\rho}^{1/2} 
\mathcal{T}_{\rho,\lambda}^{1/2} \|
\|\mathcal{T}_{\rho,\widetilde{\lambda}}^{1/2} \|
\bigg( 
\frac{2 \kappa}{m \sqrt{ \lambda} } 
+ \frac{\sqrt{c_{\gamma}}}{\sqrt{m \lambda^{\gamma}}} 
\bigg) \log \frac{4 n }{\delta}
\max_{w \in V}\big\{ 
\|\mathcal{T}^{-1/2}_{\rho,\widetilde{\lambda}}N_{k,w}\|_{H}\big\}.
\end{align}
\end{proposition}
\begin{proof}
Fix $t-1 \geq j  \geq k+1$ and $w_{t:k} \in V^{t-k+1}$. Begin by proving \eqref{equ:DistributedError:Bound1_new:1}. Expanding the norm,
\begin{align*}
& \| \mathcal{T}_{\rho}^{1/2} \Pi_{t:j+1}(\mathcal{T}_{\rho})
	(\mathcal{T}_{\rho} - \mathcal{T}_{\mathbf{x}_{w_{j}}})
	\Pi_{j-1:k+1}(\mathcal{T}_{\mathbf{x}_{w_{j-1:k+1}}})
	N_{k,w_{k}}\|_{H}\\
& 
=
\| \mathcal{T}_{\rho}^{1/2} \Pi_{t:j+1}(\mathcal{T}_{\rho})
\mathcal{T}_{\rho,\lambda}^{1/2}
\mathcal{T}_{\rho,\lambda}^{-1/2}
	(\mathcal{T}_{\rho} - \mathcal{T}_{\mathbf{x}_{w_{j}}})
	\Pi_{j-1:k+1}(\mathcal{T}_{\mathbf{x}_{w_{j-1:k+1}}})
	\mathcal{T}_{\rho,\widetilde{\lambda}}^{1/2}
\mathcal{T}_{\rho,\widetilde{\lambda}}^{-1/2}
	N_{k,w_{k}}\|_{H}\\
& \leq 
\| \mathcal{T}_{\rho}^{1/2} \Pi_{t:j+1}(\mathcal{T}_{\rho})
\mathcal{T}_{\rho,\lambda}^{1/2}\|
\|\mathcal{T}_{\rho,\lambda}^{-1/2}
	(\mathcal{T}_{\rho} - \mathcal{T}_{\mathbf{x}_{w_{j}}})\|
\|\Pi_{j-1:k+1}(\mathcal{T}_{\mathbf{x}_{w_{j-1:k+1}}})\|
\|\mathcal{T}_{\rho,\widetilde{\lambda}}^{1/2}\|
\|\mathcal{T}_{\rho,\widetilde{\lambda}}^{-1/2}
	N_{k,w_{k}}\|_{H}\\
	& \leq 
	\| \mathcal{T}_{\rho}^{1/2} \Pi_{t:j+1}(\mathcal{T}_{\rho})
\mathcal{T}_{\rho,\lambda}^{1/2}\|
\|\mathcal{T}_{\rho,\lambda}^{-1/2}
	(\mathcal{T}_{\rho} - \mathcal{T}_{\mathbf{x}_{w_{j}}})\|
\|\mathcal{T}_{\rho,\widetilde{\lambda}}^{1/2}\|
\|\mathcal{T}_{\rho,\widetilde{\lambda}}^{-1/2}
	N_{k,w_{k}}\|_{H}\\
	& \leq 
		\| \mathcal{T}_{\rho}^{1/2} \Pi_{t:j+1}(\mathcal{T}_{\rho})
\mathcal{T}_{\rho,\lambda}^{1/2}\|
\|\mathcal{T}_{\rho,\lambda}^{-1/2}
	(\mathcal{T}_{\rho} - \mathcal{T}_{\mathbf{x}_{w_{j}}})\|
\|\mathcal{T}_{\rho,\widetilde{\lambda}}^{1/2}\|
\max_{w \in V} \big\{ 
\|\mathcal{T}_{\rho,\widetilde{\lambda}}^{-1/2}N_{k,w}\|_{H}\big\},
\end{align*}
where we used, from $\eta \kappa^2 \leq 1$ and $\eta \|\mathcal{T}_{\mathbf{x}_{v}}\| \leq 1$ for any $v \in V$, that $\|\Pi_{j-1:k+1}(\mathcal{T}_{\mathbf{x}_{w_{j-1:k+1}}})\| \leq 1$  for $j \geq k+2$. 
The first operator norm is bounded  as follows by using techniques similar to those used to prove Lemma \ref{Lem:OperatorNorm}:
\begin{align}
\|\mathcal{T}^{1/2}_{\rho}
 \Pi_{t:j+1}(\mathcal{T}_{\rho}) 
 \mathcal{T}^{1/2}_{\rho,\lambda}\| 
 & \leq 
 \bigg( 
\frac{1}{e \sum_{i=j+1}^{t} \eta_i } + \bigg( \frac{\lambda}{ 2e \sum_{i=j+1}^{t} \eta_i } \bigg)^{1/2} 
\bigg) \nonumber \\
\label{equ:NormOperatorLambdaBound:1}
& \leq 
\bigg( 
\frac{1}{\sum_{i=j+1}^{t} \eta_i } + \bigg( \frac{\lambda}{ \sum_{i=j+1}^{t} \eta_i } \bigg)^{1/2} 
\bigg).
\end{align}
We proceed to construct a high-probability bound for the quantity $\| (\mathcal{T}_{\rho} + \lambda I )^{-1/2}(\mathcal{T}_{\rho} - \mathcal{T}_{\mathbf{x}_{w_{j}}})
\|$, for any  $w_{j} \in V$. For $v \in V$, let $\delta_{v} = \frac{\delta}{n}$ and apply \eqref{equ:lem:DistributedError:2} from Lemma \ref{lem:DistributedError:Bounds} to obtain the following\footnote{
  For an operator $L$ note that $\|L\| = \|L L^{\star}\|^{1/2}$ where $L^{\star}$ is the adjoint of $L$. The Hilbert-Schmidt norm bounds the operator norm as we have $\|L\|^{2} = \|L L^{\star}\| \leq \trace\big( L L^{\star}\big) = \|L\|^2_{HS}$.
} with probability at least $1-\delta_{v}$:
\begin{align*}
\| (\mathcal{T}_{\rho} + \lambda I )^{-1/2}(\mathcal{T}_{\rho} - \mathcal{T}_{\mathbf{x}_{v}})
\|
& \leq 
\| (\mathcal{T}_{\rho} + \lambda  I )^{-1/2}(\mathcal{T}_{\rho} - \mathcal{T}_{\mathbf{x}_{v}})
\|_{HS}
 \leq 
2 \kappa 
\bigg( 
\frac{2 \kappa}{m \sqrt{ \lambda} } 
+ \frac{\sqrt{c_{\gamma}}}{\sqrt{m \lambda^{\gamma}}} 
\bigg) \log \frac{4 n }{\delta}.
\end{align*}
Applying a union bound yields the following with probability at least $1-\sum_{v \in V} \delta_{v} = 1-\delta$:
\begin{align}
\label{equ:UnionHighProbability}
\| (\mathcal{T}_{\rho} + \lambda I )^{-1/2}(\mathcal{T}_{\rho} - \mathcal{T}_{\mathbf{x}_{v}})
\| 
\leq 2 \kappa 
\bigg( 
\frac{2 \kappa}{m \sqrt{ \lambda} } 
+ \frac{\sqrt{c_{\gamma}}}{\sqrt{m \lambda^{\gamma}}} 
\bigg) \log \frac{4 n }{\delta}
\quad 
\forall v \in V.
\end{align}
The result \eqref{equ:DistributedError:Bound1_new:1} then comes from plugging \eqref{equ:NormOperatorLambdaBound:1} and \eqref{equ:UnionHighProbability} into the expanded quantity at the start of the proof.

To prove \eqref{equ:DistributedError:Bound1_new:2}, fix $t-1 \geq k \geq 1$ and $w_t, w_{k} \in V$.
Expanding the norm we get 
\begin{align*}
\| \mathcal{T}_{\rho}^{1/2} (\mathcal{T}_{\rho} - \mathcal{T}_{\mathbf{x}_{w_{t}}})
N_{k,w_{k}}\|_{H}
& = 
\| \mathcal{T}_{\rho}^{1/2} 
\mathcal{T}_{\rho,\lambda}^{1/2} 
\mathcal{T}_{\rho,\lambda}^{-1/2}
(\mathcal{T}_{\rho} - \mathcal{T}_{\mathbf{x}_{w_{t}}})
\mathcal{T}_{\rho,\widetilde{\lambda}}^{1/2} 
\mathcal{T}_{\rho,\widetilde{\lambda}}^{-1/2}
N_{k,w_{k}}\|_{H}\\
& \leq 
\|\mathcal{T}_{\rho}^{1/2} 
\mathcal{T}_{\rho,\lambda}^{1/2} \|
\|\mathcal{T}_{\rho,\lambda}^{-1/2}
(\mathcal{T}_{\rho} - \mathcal{T}_{\mathbf{x}_{w_{t}}})\|
\|\mathcal{T}_{\rho,\widetilde{\lambda}}^{1/2} \|
\|\mathcal{T}_{\rho,\widetilde{\lambda}}^{-1/2}
N_{k,w_{k}}\|_{H}\\
& \leq 
\|\mathcal{T}_{\rho}^{1/2} 
\mathcal{T}_{\rho,\lambda}^{1/2} \|
\|\mathcal{T}_{\rho,\lambda}^{-1/2}
(\mathcal{T}_{\rho} - \mathcal{T}_{\mathbf{x}_{w_{t}}})\|
\|\mathcal{T}_{\rho,\widetilde{\lambda}}^{1/2} \|
\max_{w \in V}\big\{ \|\mathcal{T}^{-1/2}_{\rho,\widetilde{\lambda}}N_{k,w}\|_{H}\big\}.
\end{align*}
The result follows by using \eqref{equ:UnionHighProbability} to bound $\|\mathcal{T}_{\rho,\lambda}^{-1/2}
(\mathcal{T}_{\rho} - \mathcal{T}_{\mathbf{x}_{w_{t}}})\|$.
\end{proof}
The following proposition utilise the previous proposition to bound the summation \eqref{prop:equ:FirstOrderError}.
\begin{proposition}
\label{prop:DistributedError:Bound2_new}
Let the assumptions of Proposition \ref{prop:DistributedError:Bound1_new} hold. For any $v \in V$, with probability at least $1-\delta$ we have
\begin{align*}
\textbf{Resid. Emp. Cov. Error}
\leq 
8\kappa \bigg( \frac{2 \kappa}{m \sqrt{ \lambda} } 
+ \frac{\sqrt{c_{\gamma}}}{\sqrt{m \lambda^{\gamma}}} 
\bigg) \log \frac{4n}{\delta}
\bigg[ \textbf{B}_{1} + \textbf{B}_{2} \bigg],
\end{align*}
where
\begin{align*}
\textbf{B}_1 & = 
\|\mathcal{T}^{1/2}_{\rho} \mathcal{T}^{1/2}_{\rho,\lambda}\|
\|\mathcal{T}^{1/2}_{\rho,\widetilde{\lambda}}\| 
\eta_{t} \sum_{k=1}^{t-1} \eta_k
\max_{w \in V}\big\{ \|\mathcal{T}^{-1/2}_{\rho,\widetilde{\lambda}}N_{k,w}\|_{H}\big\},\\
\textbf{B}_{2}
& = \|\mathcal{T}^{1/2}_{\rho,\widetilde{\lambda}}\|
\sum_{k=1}^{t-2}\eta_{k} \sum_{j=k+1}^{t-1} \eta_{j}
\bigg( 
\frac{1}{ \sum_{i=j+1}^{t} \eta_i } + \bigg( \frac{\lambda}{  \sum_{i=j+1}^{t} \eta_i } \bigg)^{1/2} \bigg) 
\max_{w \in V}\big\{ \|\mathcal{T}^{-1/2}_{\rho,\widetilde{\lambda}}N_{k,w}\|_{H}\big\}.
\end{align*}
\end{proposition}
\begin{proof}
Splitting  the sum in \eqref{prop:equ:FirstOrderError} at $j=t$ and otherwise, directly applying \eqref{equ:DistributedError:Bound1_new:1} and \eqref{equ:DistributedError:Bound1_new:2}
from Proposition \ref{prop:DistributedError:Bound1_new} allows $\textbf{Resid. Emp. Cov. Error}$  to be bounded as follows:
\begin{align*}
& \textbf{Resid. Emp. Cov. Error} \\
& \leq
\eta_{t} \sum_{k=1}^{t-1} \eta_{k} \sum_{w_{t:k} \in V^{t-k+1}} 
	|\Delta(w_{t:k})| 
	\| \mathcal{T}_{\rho}^{1/2} (\mathcal{T}_{\rho} - \mathcal{T}_{\mathbf{x}_{w_{t}}})
	\Pi_{t-1:k+1}(\mathcal{T}_{\mathbf{x}_{w_{t-1:k+1}}})
	N_{k,w_{k}}\|_{H}
\\
& \quad\ +\!
	\sum_{k=1}^{t-2} \eta_{k}\!\!\!\!
	\sum_{w_{t:k} \in V^{t-k+1}} \!\!\!\!\!\!
	|\Delta(w_{t:k})|\!
	\sum_{j=k+1}^{t-1} \!\!\!\eta_{j}
	\| \mathcal{T}_{\rho}^{1/2} \Pi_{t:j+1}(\mathcal{T}_{\rho})
	(\mathcal{T}_{\rho} \!-\! \mathcal{T}_{\mathbf{x}_{w_{j}}})
	\Pi_{j-1:k+1}(\mathcal{T}_{\mathbf{x}_{w_{j-1:k+1}}})
	N_{k,w_{k}}\|_{H}\\
& \leq  
2 \kappa 
\bigg( 
\frac{2 \kappa}{m \sqrt{ \lambda} } 
+ \frac{\sqrt{c_{\gamma}}}{\sqrt{m \lambda^{\gamma}}} 
\bigg) \log \frac{4 n }{\delta} \\
& \quad \times 
\bigg [
\underbrace{ 
\|\mathcal{T}^{1/2}_{\rho} \mathcal{T}^{1/2}_{\rho,\lambda}\|
\|\mathcal{T}^{1/2}_{\rho,\widetilde{\lambda}}\| 
\eta_{t} \sum_{k=1}^{t-1} \eta_{k} 
\max_{w \in V}\big\{ \|\mathcal{T}^{-1/2}_{\rho,\widetilde{\lambda}}N_{k,w}\|_{H}\big\}
}_{\textbf{B}_{1}}
\sum_{w_{t:k} \in V^{t-k+1}} 
	|\Delta(w_{t:k})| \\
	& \quad\quad \
	+
	\underbrace{ 
\|\mathcal{T}^{1/2}_{\rho,\widetilde{\lambda}}\| 
	\sum_{k=1}^{t-2} \eta_{k}
	\sum_{j=k+1}^{t-1} \eta_{j}
	\bigg( 
\frac{1}{ \sum_{i=j+1}^{t} \eta_i } + \bigg( \frac{\lambda}{  \sum_{i=j+1}^{t} \eta_i } \bigg)^{1/2} \bigg) 
\max_{w \in V}\big\{ \|\mathcal{T}^{-1/2}_{\rho,\widetilde{\lambda}}N_{k,w}\|_{H}\big\}
}_{\textbf{B}_{2}}
	\\
	& 
\hspace{9cm}
	\times 
	\sum_{w_{t:k} \in V^{t-k+1}} 
	|\Delta(w_{t:k})| \bigg]
	\nonumber.
\end{align*}
The result  is then arrived at by applying the following bound for the summation $\sum_{w_{t:k} \in V^{t-k+1}} |\Delta(w_{t:k})|$  for each $k \leq t$:
\begin{align*}
& \sum_{w_{t:k} \in V^{t-k+1}} |\Delta(w_{t:k})| = 
\sum_{w_{t:k} \in V^{t-k+1}} \bigg|P_{vw_{t:k}} - \frac{1}{n^{t-k+1}}\bigg|\\
&= 
\sum_{\substack{ w_{t:k} \in V^{t-k+1} \\
P_{vw_{t:k}} \geq n^{-(t-k+1)} }}
\bigg( 
P_{vw_{t:k}} - \frac{1}{n^{t-k+1}}\bigg) 
- 
\sum_{\substack{ w_{t:k} \in V^{t-k+1} \\
P_{vw_{t:k}} < n^{-(t-k+1)} }}
\bigg( P_{vw_{t:k}} - \frac{1}{n^{t-k+1}} \bigg) \leq 4.
\end{align*}
\end{proof}

Given Proposition \ref{prop:DistributedError:Bound2_new} we can now plug in a high-probability bound for $\max_{w \in V}\big\{ \|\mathcal{T}^{-1/2}_{\rho,\widetilde{\lambda}}
N_{k,w}\|_{H}\big\} $
and bound the resulting summations. This is summarised in the following lemma.
\begin{lemma}
\label{lem:HighProbBoundTermb}
Let the assumptions of Proposition  \ref{prop:DistributedError:Bound1_new} hold with $0 \leq \theta \leq 3/4$, $0 \leq \lambda \leq \|\mathcal{T}_{\rho}\|$ and 
$ 0 \leq \widetilde{\lambda} \leq \|\mathcal{T}_{\rho}\|$. Given $\delta \in (0,1)$, the following holds with probability  at least $1-\delta$:
\begin{align*}
& \textbf{Resid. Emp. Cov. Error}\\
& \leq
\widetilde{b}_1
\frac{ 
\log^2 \frac{8n}{\delta} \log(t) 
}{ m \sqrt{ \big( (m \lambda) \wedge \lambda^{\gamma}  \big) 
\big(  (m \widetilde{\lambda} ) \wedge \widetilde{ \lambda} ^{\gamma} \big) }   
} ( 1\vee (\eta t^{1-\theta}) \vee \sqrt{ \lambda } (\eta t^{1-\theta})^{3/2} \vee (t^{-1}(\eta t^{1-\theta})^{2}) ),
\end{align*}
where $\widetilde{b}_1 =  
\frac{128 \kappa (R \kappa^{2r} + \sqrt{M})(2 \kappa + \sqrt{2 \sqrt{\nu} c_{\gamma}})^2 
\|\mathcal{T}_{\rho}\|^{1/2}  (4 + \|\mathcal{T}_{\rho}\| )
}{(1-\theta)} 
$.
\end{lemma}
\begin{proof}
Consider Proposition \ref{prop:DistributedError:Bound2_new} with $\frac{\delta }{2}$, so the following holds with probability at least  $1-\frac{\delta}{2}$
\\
\begin{align*}
\textbf{Resid. Emp. Cov. Error}
& \leq 
8\kappa \bigg( \frac{2 \kappa}{m \sqrt{ \lambda } } 
+ \frac{\sqrt{c_{\gamma}}}{\sqrt{m \lambda^{\gamma}}} 
\bigg) \log \frac{8n}{\delta}
( \textbf{B}_{1} + \textbf{B}_{2})  \\
& \leq 
\frac{  8\kappa (2\kappa + \sqrt{2 \sqrt{\nu} c_{\gamma}})
}{
\sqrt{ (m \lambda ) \wedge \lambda^{\gamma} } 
} \frac{ \log \frac{8n}{\delta}}{\sqrt{m}}
( \textbf{B}_{1} + \textbf{B}_{2} ),
\end{align*}
where we used that $\nu \geq 1$.
Proceed to bound both $\textbf{B}_{1}$ and $\textbf{B}_{2}$. Start by  constructing a high-probability bound for the term $\max_{w \in V}\big\{ \|\mathcal{T}^{-1/2}_{\rho,\widetilde{\lambda}}
N_{k,w}\|_{H}\big\}$ $k=1,\dots,t$. 
For $v \in V$, let $\delta_{v}^{\prime} = \frac{\delta}{2n}$. Lemma \ref{lem:DistributedError:Bounds} states with probability at least $1-\delta_{v}^{\prime}$ the following holds for any $k \in \mathbb{N}$:
\begin{align*}
\|\mathcal{T}^{-1/2}_{\rho,\widetilde{\lambda}}N_{k,v}\| \leq 
4(R \kappa^{2r} + \sqrt{M})
\bigg(
\frac{\kappa}{m \sqrt{\widetilde{\lambda}}} 
+
\frac{\sqrt{2 \sqrt{\nu} c_{\gamma}}}{\sqrt{ m \widetilde{\lambda}^{\gamma}}} 
\bigg)
\log\frac{8n}{\delta}.
\end{align*}
Applying a union bound so the following holds  with probability at least  $1-\sum_{v \in V} \delta^{\prime}_{v} = 1-\frac{\delta}{2}$  for any $k \in \mathbb{N}$:
\begin{align}
\max_{w \in V}\big\{ \|\mathcal{T}^{-1/2}_{\rho,\widetilde{\lambda}}
N_{k,w}\|_{H}\big\}
& \leq 
4(R \kappa^{2r} + \sqrt{M})
\bigg(
\frac{\kappa}{m \sqrt{\widetilde{\lambda} }} 
+
\frac{\sqrt{2 \sqrt{\nu} c_{\gamma}}}{\sqrt{ m \widetilde{\lambda}^{\gamma}}} 
\bigg)
\log\frac{8n}{\delta} \nonumber \\
& \leq 
\label{equ:DistributedError:MaximumProbBound}
\frac{ 4(R \kappa^{2r} + \sqrt{M})(2\kappa + \sqrt{2 \sqrt{\nu} c_{\gamma}})}{
\sqrt{ (m \widetilde{\lambda} ) \wedge \widetilde{\lambda}^{\gamma} }} \frac{\log \frac{8n}{\delta}}{\sqrt{ m}},
\end{align}
where we used that $\kappa \geq 1$. The terms $\textbf{B}_1$ and $\textbf{B}_2$ are now bounded in the following two paragraphs. 

\paragraph{Term $\textbf{B}_1$}
Using the high-probability bound \eqref{equ:DistributedError:MaximumProbBound}, the following holds with probability at least $1-\frac{\delta}{2}$:
\begin{align*}
\textbf{B}_{1} 
& \leq 
\|\mathcal{T}^{1/2}_{\rho} \mathcal{T}^{1/2}_{\rho,\lambda}\|
\|\mathcal{T}^{1/2}_{\rho,\widetilde{\lambda}}\|  
\frac{ 4(R \kappa^{2r} + \sqrt{M})(2\kappa + \sqrt{2 \sqrt{\nu} c_{\gamma}})}{
\sqrt{ (m \widetilde{\lambda}) \wedge \widetilde{\lambda}^{\gamma} }  
} \frac{\log \frac{8n}{\delta}}{\sqrt{ m}}
\eta_{t} \sum_{k=1}^{t-1} \eta_{k}\\
& \leq 
\|\mathcal{T}^{1/2}_{\rho} \mathcal{T}^{1/2}_{\rho,\lambda}\|
\|\mathcal{T}^{1/2}_{\rho,\widetilde{\lambda}}\|  
\frac{ 4(R \kappa^{2r} + \sqrt{M})(2\kappa + \sqrt{2 \sqrt{\nu} c_{\gamma}})}{
\sqrt{ (m \widetilde{\lambda} ) \wedge \widetilde{\lambda}^{\gamma} }  (1-\theta) 
} \frac{\log \frac{8n}{\delta}}{\sqrt{ m}}t^{-1} (\eta t^{1-\theta})^2,
\end{align*}
where we have applied the integral bound t $\sum_{k=1}^{t-1} k^{-\theta} \leq \frac{t^{1-\theta}}{1-\theta}$, see for instance \cite[Lemma 12]{lin2017optimal},  on the following summation:
\begin{align*}
\eta_{t} \sum_{k=1}^{t-1} \eta_{k}
 & = \eta^2 t^{-\theta} \sum_{k=1}^{t-1} k^{-\theta}
 \leq \frac{ \eta^2 }{1-\theta} t^{1-2\theta}
 = \frac{ t^{-1} (\eta t^{1-\theta})^2}{1-\theta}.
\end{align*}

\paragraph{Term $\textbf{B}_2$}
Similarly, using the high-probability bound \eqref{equ:DistributedError:MaximumProbBound}, the following holds with probability at least $1-\frac{\delta}{2}$:
\begin{align*}
\textbf{B}_2  
& \leq 
\|\mathcal{T}^{1/2}_{\rho,\widetilde{\lambda}}\| 
\frac{ 4(R \kappa^{2r} + \sqrt{M})(2\kappa + \sqrt{2 \sqrt{\nu} c_{\gamma}})}{
\sqrt{ (m \widetilde{\lambda} ) \wedge \widetilde{\lambda} ^{\gamma} }  } \frac{\log \frac{8n}{\delta}}{\sqrt{ m}} \\
& \quad\ \times 
\sum_{k=1}^{t-2}\eta_{k} \sum_{j=k+1}^{t-1} \eta_{j}
\bigg( 
\frac{1}{ \sum_{i=j+1}^{t} \eta_i } + \bigg( \frac{ \lambda}{  \sum_{i=j+1}^{t} \eta_i } \bigg)^{1/2} \bigg).
\end{align*} 
We proceed to bound the remaining terms by utilising results from Section \ref{sec:UsefulIneq}. Firstly, switching the order of sums and applying an integral bound yields 
\begin{align}
\sum_{k=1}^{t-2} \eta_{k} \sum_{j=k+1}^{t-1} 
\frac{\eta_{j} }{ \sum_{i=j+1}^{t} \eta_i }
& = \eta \sum_{k=1}^{t-2} k^{-\theta} 
\sum_{j=k+1}^{t-1} 
\frac{ j^{-\theta} }{ \sum_{i=j+1}^{t} i^{-\theta} }
\nonumber \\
&= \eta \sum_{j=2}^{t-1}
\frac{ j^{-\theta} }{ \sum_{i=j+1}^{t} i^{-\theta} }
\sum_{k=1}^{j-1}
k^{-\theta}\nonumber \\
& \leq \label{equ:PartialSumBound}
\frac{ \eta }{1-\theta}
\sum_{j=2}^{t-1}
\frac{ j^{-\theta}(j-1)^{1-\theta} }{ \sum_{i=j+1}^{t} i^{-\theta} }.
\end{align}
At this point use $\sum_{i=j+1}^{t} i^{-\theta} \geq t^{-\theta}(t-j)$ as well as Lemma \ref{lem:SeriesBound} to obtain
\begin{align*}
\sum_{j=2}^{t-1}
\frac{ j^{-\theta}(j-1)^{1-\theta} }{ \sum_{i=j+1}^{t} i^{-\theta} }
\leq 
t^{\theta} \sum_{j=2}^{t-2} 
\frac{(j-1)^{1-2\theta}}{t-j}
\leq 4 t^{\theta} t^{-\min(2\theta-1,1)} \log(t)
= 
4 t^{1-\theta} \log(t).
\end{align*}
For the second term follow the steps to \eqref{equ:PartialSumBound} and use Lemma \ref{lem:SeriesBoundHalf} as follows:
\begin{align*}
\sum_{k=1}^{t-2} \eta_{k} \sum_{j=k+1}^{t-1} 
\frac{\eta_{j} }{ \big( \sum_{i=j+1}^{t} \eta_i \big)^{1/2} }
& \leq 
\frac{\eta^{3/2} t^{\theta/2} }{1-\theta}
\sum_{j=2}^{t-1}\frac{(j-1)^{1-2\theta}}{(t-j)^{1/2}}\\
& \leq 
\frac{4 \eta^{3/2} t^{\theta/2} }{1-\theta} t^{\max(3/2 - 2\theta,0)}\\
& = 
\frac{4 \eta^{3/2} }{1-\theta} t^{\max(3(1-\theta)/2,\theta/2)}.
\end{align*}
This results in the following bound for $\textbf{B}_{2}$, which holds with probability at least $1-\frac{\delta}{2}$:
\begin{align*}
\textbf{B}_2  
& \leq 
\|\mathcal{T}^{1/2}_{\rho,\widetilde{\lambda}}\| 
\frac{ 4(R \kappa^{2r} + \sqrt{M})(2\kappa + \sqrt{2 \sqrt{\nu} c_{\gamma}})}{
\sqrt{ (m \widetilde{\lambda} ) \wedge \widetilde{\lambda} ^{\gamma} }  (1-\theta)} 
\frac{\log \frac{8n}{\delta} \log(t) }{\sqrt{ m}} 
\big( 
4 \eta t^{1-\theta} + 4\sqrt{ \lambda } \big(\eta t^{ \max(1-\theta,\theta/3)}\big)^{3/2}
\big).
\end{align*} 

The final bound arises by bringing everything together with a union bound implying it holds with probability at least  $1-\frac{\delta}{2} - \frac{\delta}{2} = 1-\delta$. Constants are then cleaned up using $\lambda \leq \|\mathcal{T}_{\rho}\|$ as well as $\widetilde{\lambda} \leq \|\mathcal{T}_{\rho}\|$ to say
$\|\mathcal{T}^{1/2}_{\rho} \mathcal{T}^{1/2}_{\rho,\lambda}\|\|\mathcal{T}^{1/2}_{\rho,\widetilde{\lambda}}\| \leq 4\|\mathcal{T}_{\rho}\|^{3/2}$ and 
$\|\mathcal{T}_{\rho,\widetilde{\lambda}}^{1/2}\|  \leq 2 \|\mathcal{T}_{\rho}\|^{1/2}$. 
\end{proof}

\subsubsection{Network Error bound}
\label{Sec:Network Error bound}
In this section we bring together the bounds developed in the previous two sections for the \textbf{Population Covariance Error} term and \textbf{Residual Empirical Covariance Error} term to construct the final bound on the Network Term as presented in the following theorem.
\begin{theorem}
\label{thm:DecentralisedErrorBound}
Let Assumptions \ref{Assumption:Moments}, \ref{Assumption:Source}, \ref{Assumption:Capacity} hold with $r \geq 1/2$, and $\eta_t = \eta t^{-\theta}$ for $t \in \mathbb{N}$ with $\eta \kappa^2 \leq 1$ and $\theta \in (0,3/4)$. Assume $t/2 \geq  \lceil \frac{(r+1)\log(t)}{1-\sigma_2}\rceil  =: t^{\star}$ 
The following bound holds for any $v \in V$, $\alpha \in [0,1/2]$ and $\gamma^{\prime} \in [1,\gamma]$:
\begin{align*}
& \E[
\| \mathcal{S}_{\rho} ( \omega_{t+1,v}  - \xi_{t+1,v}) \|_{\rho}^2]
 \leq 
2 
\frac{ \widetilde{a} \log^2(4n) \log^2(t^{\star})}{m} 
\Big(  \eta^2 t^{-2r} 
\vee (m^{-1} (\eta t^{\star})^{1 + 2\alpha})  \vee (\eta t^{\star})^{\gamma^{\prime} + 2 \alpha}
\Big)\\
& + 
 2 \widetilde{b}_2
\frac{ \log^4 (8n) \log^2 (t) }{m^2 } 
\Big( 1 \vee (\eta t^{1-\theta})^{2 } \vee (t^{-2}(\eta t^{1-\theta})^{4 }) \Big)
\Big( (m^{-1}\eta t^{1-\theta}) \vee (\eta t^{1-\theta})^{\gamma} \Big),
\end{align*}
where $\widetilde{b}_2 = 
64 \frac{(\|\mathcal{T}_{\rho}\| + 1)^2}{ 
(  \|\mathcal{T}_{\rho}\|  \wedge \|\mathcal{T}_{\rho}\|^{\gamma} 
 )^2
} \widetilde{b}^2_1$ with $\widetilde{b}_{1}$ defined as in Theorem \ref{lem:HighProbBoundTermb}
and $\widetilde{a}$ defined as in Lemma \ref{lem:DistributedError:Terma}.
\end{theorem}
\begin{proof}
Use decomposition \eqref{equ:DistributedError:Breakdown}. Taking the expectation, note that the first term $\E[ (\textbf{Pop. Cov. Error})^2 ]$ is controlled by Lemma \ref{lem:DistributedError:Terma}. We now proceed to control the term $\E[ (\textbf{Resid. Emp. Cov. Error})^2 ]$.

Begin by using the high-probability bound for $\textbf{Resid. Emp. Cov. Error}$ in Lemma \ref{lem:HighProbBoundTermb}, with $\widetilde{\lambda}  = \|\mathcal{T}_{\rho}\|$ and $\lambda = \|\mathcal{T}_{\rho}\| (\eta t^{1-\theta})^{-1}$.
The following upper bound holds for the quantity that appears in Lemma \ref{lem:HighProbBoundTermb}:
\begin{align*}
\frac{1}{ (m \lambda) \wedge \lambda^{\gamma} }
& = 
\frac{1}{ (\|\mathcal{T}_{\rho}\| m  (\eta t^{1-\theta})^{-1}) \wedge 
( \|\mathcal{T}_{\rho}\|^{\gamma}(\eta t^{1-\theta})^{-\gamma}) }\\
& \leq 
\frac{1}{\|\mathcal{T}_{\rho}\| \wedge \|\mathcal{T}_{\rho}\|^{\gamma}} 
\Big( ( m^{-1} (\eta t^{1-\theta})) \vee (\eta t^{1-\theta})^{\gamma} \Big).
\end{align*}
Plugging the above into Lemma \ref{lem:HighProbBoundTermb}  for the \textbf{Resid. Emp. Cov. Error} allows  the expectation to be bounded with Lemma \ref{Lem:TailBoundToExpectation}:
 \begin{align*}
 & \E[ (\textbf{Resid. Emp. Cov. Error} )^{2} ] \\
 & \leq 
 \widetilde{b}_{1}^{2} \frac{(\|\mathcal{T}_{\rho}\| + 1)^2}{ 
 \big( \|\mathcal{T}_{\rho}\| \wedge \|\mathcal{T}_{\rho}\|^{\gamma} \big)^2 }
 \frac{ \log^2 (t) }{m^2}
\Big( 1 \vee (\eta t^{1-\theta})^{2} \vee (t^{-2}(\eta t^{1-\theta})^{4}) \Big)
 \Big( (m^{-1} \eta t^{1-\theta}) \vee (\eta t^{1-\theta})^{\gamma} \Big)\\
 &\quad\times \int_{0}^{1} \log^4 \frac{8n}{\delta} d\delta.
 \end{align*}
 The result is arrived at by using $ \int_{0}^{1} \log^4 \frac{8n}{\delta} d\delta \leq 64 \log^4(8n)$ and bringing together the two bounds for
 $\E[ (\textbf{Pop. Cov. Error})^2 ]$ and $\E[ (\textbf{Resid. Emp. Cov. Error})^2 ]$.
 \end{proof}

\subsection{Final Bound}
\label{sec:Appendix:ConstructingFinalBound}
In this section we bring together the bounds from the previous sections to construct the final bounds in Theorem \ref{thm:MainResult} and Theorem \ref{Cor:Main} in the main body of the work. The main result is the following.

\begin{theorem}
\label{thm:Appendix:Main}
Let Assumptions \ref{Assumption:Moments}, \ref{Assumption:Source}, \ref{Assumption:Capacity} hold with $r \geq 1/2$ and $\eta_{t} = \eta t^{-\theta}$ for all $t  \in \mathbb{N}$  with $\eta \kappa^2 \leq 1$ $\theta \in (0,3/4)$. 
The following holds for all $ t/2  \geq  
\lceil \frac{ (r+1) \log(t) }{1-\sigma_2}\rceil =:  t^{\star} $, any $v \in V$, $\alpha \in [0,1/2]$ and $\gamma^{\prime} \in [1,\gamma]$:
\begin{align*}
& \E[\mathcal{E}(\omega_{t+1,v})]
- \inf_{\omega \in H}\mathcal{E}(\omega) \leq
2 R^2  (\eta t^{1-\theta})^{-2r}\\
& + 
\widetilde{d}_{4} 
(nm)^{-2r/(2r+\gamma)}
\Big( 
1 \vee (nm)^{-2/(2r+\gamma)} (\eta t^{1-\theta} )^2 \vee t^{-2}(\eta t^{1-\theta})^{2} 
\Big) \log^2(t) \\
& + 
8
\frac{ \widetilde{a} \log^2(4n) \log^2(t^{\star})}{m} 
\Big(  \eta^2 t^{-2r} 
\vee (m^{-1} (\eta t^{\star})^{1 + 2\alpha})  \vee (\eta t^{\star})^{\gamma^{\prime} + 2 \alpha})
\Big)\\
& + 
8 \frac{\widetilde{b}_{2} \log^{4}(8n) \log^2(t) }{m^2}
\Big(
 1 \vee (\eta t^{1-\theta})^{2} \vee t^{-2}(\eta t^{1-\theta})^{4}
 \Big)
  \Big( (m^{-1} \eta t^{1-\theta}) \vee (\eta t^{1-\theta})^{\gamma}\Big),
\end{align*}
where $\widetilde{d}_{4} = 4 \big( \frac{2r+\gamma}{2r+\gamma - 1}\big)^{2} \widetilde{d}_{3}^2$ with $\widetilde{d}_{3}$ defined as in Theorem \ref{thm:SampleVarianceBound}. 
\end{theorem}
\begin{proof}
Begin with the decomposition in Proposition \ref{Prop:ErrorDecomp} and take the expectation  $\E[\,\cdot\,]$. Plug in the bounds for each term proven in the previous sections, i.e.\ 
Proposition \ref{Prop:BiasBound} for the Bias, Theorem \ref{thm:SampleVarianceBound} with $p = 1/(2r+\gamma)$ for the Sample Variance term and Theorem \ref{thm:DecentralisedErrorBound} for the Network Error term. 
\end{proof}
Theorem \ref{thm:MainResult} follows directly from Theorem \ref{thm:Appendix:Main}. 
\begin{proof}[Proof of Theorem \ref{thm:MainResult}]
Consider Theorem \ref{thm:Appendix:Main} with constants  
\begin{align*}
    q_1 & = 2 R^2 \\
    q_2 & = \widetilde{d}_4 \\
    q_3 & = 16 \widetilde{a}( \log^2(4) + 1) \\
    q_4 & = 24 \widetilde{b}_2 (\log^2(8) + 1)^2,
\end{align*}
where the sample variance constant $\widetilde{d}_{4}$ is defined in Theorem \ref{thm:Appendix:Main}, the first network error constant $\widetilde{a}$ is defined in Lemma \ref{lem:DistributedError:Terma}, and the second network error constant $\widetilde{b}_{2}$ is defined in Theorem \ref{thm:DecentralisedErrorBound}.
\end{proof}

We now go on to prove Theorem \ref{Cor:Main}. 
\begin{proof}[Proof of Theorem \ref{Cor:Main}]
Consider the setting of Theorem \ref{thm:Appendix:Main}  with  $\theta = 0$.
Begin by setting
$$t = \Big\lceil (nm)^{1/(2r+\gamma)}
\Big[ 
\frac{1}{1-\sigma_2}
\Big( \frac{n^{r}}{m^{r+\gamma}}\Big)^{2/((1+2\alpha)(2r+\gamma))}  \vee 
\frac{1}{1-\sigma_2}
\Big( \frac{n^{2r}}{m^{\gamma}}\Big)^{1/((\gamma^{\prime}+2\alpha)(2r+\gamma))}
\vee 1 \Big] 
\Big \rceil
$$ 
and
$ \eta = \kappa^{-2} (nm)^{1/(2r+\gamma)}/t $.
It is clear that $\eta t = \kappa^{-2} (nm)^{1/(2r+\gamma)}$.
We  proceed to show that this choice of iterations $t$ and step size $\eta$  ensures each of the terms in the bound of Theorem \ref{thm:Appendix:Main} are of order $\widetilde{O}((nm)^{-2r/(2r+\gamma)})$.

The Bias term is 
\begin{align*}
2 R^2  (\eta t)^{-2r} =  2R^2 \kappa^{4r} (nm)^{-2r/(2r+\gamma)} .
\end{align*}
The Sample Variance term is bounded as follows:
\begin{align*}
& \widetilde{d}_{4} 
(nm)^{-2r/(2r+\gamma)}
\Big( 
1 \vee (nm)^{-2/(2r+\gamma)} (\eta t )^2 \vee t^{-2}(\eta t)^{2} 
\Big) \log^2(t) \\
&  \leq 4 \kappa^{-4} \widetilde{d}_{4} 
(nm)^{-2r/(2r+\gamma)}  \log^2(t).
\end{align*}
The first Network Error term is bounded in three parts aligning with the three terms within the quantity $m^{-1} (  \eta^2 t^{-2r} 
\vee (m^{-1} (\eta t^{\star})^{1 + 2\alpha})  \vee (\eta t^{\star})^{\gamma^{\prime} + 2 \alpha}))$. 
Firstly, as $t \geq (nm)^{1/(2r+\gamma)}$ and $\eta \leq 1/\kappa^2$ we get 
$\eta^2 t^{-2r} \leq \kappa^{-4} (nm)^{-2r/(2r+\gamma)}$. 
Secondly, from $t \geq (nm)^{1/(2r+\gamma)} \frac{1}{1-\sigma_2} \Big( \frac{n^{r}}{m^{r+\gamma}} \Big)^{2/((1+2\alpha)(2r+\gamma))}$
ensuring
$\eta \leq \kappa^{-2} (1-\sigma_2)  \Big( \frac{m^{r+\gamma}}{n^{r}}\Big)^{2/((1+2\alpha)(2r+\gamma))}$ we get 
\begin{align*}
    \frac{ (\eta t^{\star})^{1+2\alpha} }{m^2} 
    &\leq 
    (\kappa^{-2} 2(r+1) \log(t))^{1+2\alpha}
    \frac{ m^{2(r+\gamma)/(2r+\gamma) - 2}}{n^{2r/(2r+\gamma)}}\\
    &= 
    (\kappa^{-2} 2(r+1) \log(t))^{1+2\alpha}
    (nm)^{-2r/(2r+\gamma)}.
\end{align*}
Thirdly, from $t \geq (nm)^{1/(2r+\gamma)} \frac{1}{1-\sigma_2} \Big( \frac{n^{2r}}{m^{\gamma}} \Big)^{1/((\gamma^{\prime} + 2 \alpha)(2r+\gamma))}$ we have \\
$\eta \leq \kappa^{-2} (1-\sigma_2) \Big( \frac{m^{\gamma}}{n^{2r}}\Big)^{1/((\gamma^{\prime}+2\alpha)(2r+\gamma))} $
and so
\begin{align*}
    \frac{ (\eta t^{\star})^{\gamma^{\prime}+2\alpha} }{m}
    &\leq 
    (\kappa^{-2} 2(r+1) \log(t))^{\gamma^{\prime}+2\alpha}
    \frac{m^{\gamma/(2r+\gamma) - 1}}{n^{2r/(2r+\gamma)}}\\
    &= 
    (\kappa^{-2} 2(r+1) \log(t))^{\gamma^{\prime}+2\alpha}(nm)^{-2r/(2r+\gamma)}.
\end{align*}
Using the above three bounds we arrive at the first Network term being $\widetilde{O}((nm)^{-2r/(2r+\gamma)})$.

Now consider the second Network Error term. Since $\eta t = \kappa^{-2} (nm)^{1/(2r+\gamma)}$ and $m \geq n^{\frac{2r + 2 + \gamma}{2r+\gamma - 2}} \geq n^{\frac{1-\gamma}{2(r+\gamma) -1}}$ we have
\begin{align*}
     \Big(
 1 \vee (\eta t )^{2} \vee t^{-2}(\eta t)^{4}
 \Big) 
 \Big( \big( m^{-1} (\eta t)\big) \vee (\eta t)^{\gamma}\Big)
 \leq \Big(
 1 \vee (\eta t)^{2+\gamma} \vee t^{-2}(\eta t)^{4+\gamma}
 \Big).
\end{align*}
The second Network Error term then becomes, due to $t \geq (nm)^{1/ (2r+\gamma)}$,
\begin{align*}
& 8 \frac{\widetilde{b}_{2} \log^{4}(8n) \log^2(t) }{m^2}
\Big(
 1 \vee (\eta t)^{2+\gamma} \vee t^{-2}(\eta t)^{4+\gamma}
 \Big)\\
&   \leq
 8 (\kappa^{-2})^{2+\gamma} \widetilde{b}_{2} \log^{4}(8n) \log^2(t) 
  \frac{ (nm)^{(2+\gamma)/(2r+\gamma) }} {m^2}.
\end{align*}
For this quantity to be $\widetilde{O}((nm)^{-2r/(2r+\gamma)})$ we require $\frac{ (nm)^{(2+\gamma)/(2r+\gamma) }} {m^2} \leq 
  (nm)^{-2r/(2r+\gamma)}$ 
which is satisfied for 
$
m \geq n^{(2r+\gamma + 2)/(2r+\gamma -2)}
$.
Now ensure $\frac{t}{\log(t)} \geq 2 \frac{(1+r)}{1-\sigma_2}$. Note the previous requirements on the iterations $t$ imply 
$$
	t \geq \frac{ (nm)^{1/(2r+\gamma)}}{1-\sigma_2}
	\frac{ n^{2r/(2r+\gamma)} }{ m^{\gamma/(2r+\gamma)} } 
	\geq \frac{n^{(2r+1)/2r+\gamma}}{1-\sigma_2} \geq \frac{ n}{1-\sigma_2}  .
$$
And since $x \rightarrow x/(\log(x))$ is increasing for $x \geq 1$, the requirement $t \geq 2 \frac{ (1+r)\log(t)}{(1-\sigma_2)}$ is satisfied by  
$\frac{n }{\log ( \frac{n}{1-\sigma_2} ) } \geq 2(1+r)$.  

Now, consider choosing $\gamma^{\prime} \in [1,\gamma]$ and $\alpha \in [0,1/2]$  to minimise the number of iterations $t$. Consider the two cases   $m \geq n^{2r/\gamma}$ and $m \leq n^{2r/\gamma}$. When $m \geq n^{2r/\gamma}$  we have both $\frac{n^{2r}}{m^{\gamma}} \leq 1$ and $\frac{n^{r}}{m^{r+\gamma}} \leq 1$ so the number of iterations $t$ required is minimised by picking $\gamma^{\prime} = \gamma$  and $\alpha = 0$. Since   $2(r+\gamma) \geq 1 $
we get $\frac{ n^{2r} }{ m^{2(r+\gamma)}} \leq \frac{n^{2r/\gamma}}{m}$ and the number of iterations becomes  \\
$t = (nm)^{1/(2r+\gamma)} 
\Big[ \Big( \frac{1}{1-\sigma_2} 
\Big( \frac{n^{2r/\gamma}}{m} \Big)^{1/(2r+\gamma)}\Big) \vee 1 \Big] = (nm)^{1/(2r+\gamma)} \Big[ 
\Big( \frac{(nm)^{2r/(2r+\gamma)}}{m (1-\sigma_2)^{\gamma}} \Big)^{1/\gamma} \vee 1\Big]$. When $\frac{n^{2r}}{m^{\gamma}} \geq 1$, the number of iterations $t$ required is minimised by: setting $\gamma^{\prime} = 1$, noting $\frac{n^{2r}}{ m^{2(r+\gamma)}} \leq \frac{n^{2r}}{m^{\gamma}}$ and further picking $\alpha = 1/2$. It is clear in this case that  the number of iterations required becomes $t  = (nm)^{1/{2r+\gamma}} \frac{1}{1-\sigma_2} \Big( \frac{ n^{r}}{m^{\gamma/2}} \Big)^{1/(2r+\gamma)} =  (nm)^{1/(2r+\gamma)} \frac{ (nm)^{r/(2r+\gamma)}}{\sqrt{m}(1-\sigma_2)}$.

\end{proof}

\subsection{Useful inequalities}
\label{sec:UsefulIneq}
In this section we collect useful inequalities used within the proofs. 
\begin{lemma}
\label{lem:SeriesBound}
The following holds for $q \in \mathbb{R}$ and $t \in \mathbb{N}$ with $t \geq 3$:
\begin{align*}
\sum_{k-1}^{t-1} \frac{1}{t-k} k^{-q} \leq 2 t^{-\min(q,1)}(1+\log(t)).
\end{align*}
\end{lemma}
\begin{proof}
See Lemma 14 in \cite{lin2017optimal}.
\end{proof}
\begin{lemma}
\label{lem:SeriesBoundHalf}
The following holds for $q \in \mathbb{R}$ and $t \in \mathbb{N}$ with $t \geq 3$:
\begin{align*}
\sum_{k-1}^{t-1} \frac{1}{(t-k)^{1/2}} k^{-q} \leq 
4 t^{\max(1/2-q,0)}.
\end{align*}
\end{lemma}
\begin{proof}
Begin with 
\begin{align*}
\sum_{k=1}^{t-1} \frac{1}{(t-k)^{1/2}} k^{-q}
& \leq 
t^{\max(1/2-q,0)} \sum_{k=1}^{t-1} \frac{1}{(t-k)^{1/2} k^{1/2}}.
\end{align*}
Suppose $t$ is even. The bound arises by splitting the sum and using the integral bounds 
\begin{align*}
\sum_{k=1}^{t/2} \frac{1}{(t-k)^{1/2} k^{1/2}}
\leq
\frac{\sqrt{2}}{t^{1/2}} 
\sum_{k=1}^{t/2} \frac{1}{k^{1/2}}
\leq 
\frac{\sqrt{2}}{t^{1/2}}  \bigg[ 1 + \int_{1}^{t/2} \!\! x^{-1/2} dx\bigg]
= 
\frac{\sqrt{2}}{t^{1/2}}  \bigg[ 1 + 2\bigg(\sqrt{\frac{t}{2}} - 1\bigg) \bigg]
\leq 2,
\end{align*}
 and
\begin{align*}
\sum_{k=t/2+1}^{t-1} \frac{1}{(t-k)^{1/2} k^{1/2}}
 &\leq 
\sqrt{\frac{2}{t}}
\sum_{k=t/2+1}^{t-1} \frac{1}{(t-k)^{1/2}}
 \leq  
 \sqrt{\frac{2}{t}}
 \bigg[ 1+ \int_{t/2+1}^{t-1} (t-x)^{-1/2} dx \bigg] \\
 & = 
 \sqrt{\frac{2}{t}}\bigg[1  + 2 \bigg(  \sqrt{\frac{t}{2} -1 } - 1\bigg) \bigg] \leq 2.
\end{align*}
If $t$ is odd, follow the steps above and split the sum at $k=(t-1)/2$ and $k=(t-1)/2 + 1$.
\end{proof}

\end{document}